\documentclass{article}

    \PassOptionsToPackage{numbers, compress}{natbib}
    \usepackage[final]{neurips_2024}

\usepackage[utf8]{inputenc} % allow utf-8 input
\usepackage[T1]{fontenc}    % use 8-bit T1 fonts
\usepackage{hyperref}       % hyperlinks
\usepackage{url}            % simple URL typesetting
\usepackage{booktabs}       % professional-quality tables
\usepackage{amsfonts}       % blackboard math symbols
\usepackage{nicefrac}       % compact symbols for 1/2, etc.
\usepackage{microtype}      % microtypography
\usepackage{xcolor}         % colors
\usepackage{graphicx}
 \graphicspath{{figures/}} % set image path
\usepackage{float}
\newfloat{figtab}{htb}{fgtb}
\makeatletter
  \newcommand\figcaption{\def\@captype{figure}\caption}
  \newcommand\tabcaption{\def\@captype{table}\caption}

\usepackage{amsmath}
\usepackage{amssymb}
\usepackage{amsthm}
\usepackage{subfigure}
\usepackage{wrapfig}        % figure in words
\usepackage{algorithm}
\usepackage{algorithmic}
\usepackage{xcolor}
\usepackage{CJKutf8}

\usepackage{multirow}
\usepackage{bbm}
\usepackage[capitalize,noabbrev]{cleveref}
\usepackage{adjustbox}

\newtheorem{theorem}{Theorem}
\newtheorem{lemma}{Lemma}

\newtheorem{definition}{Definition}

\newtheorem{assumption}{Assumption}

\hypersetup{
	colorlinks=true
}

\newcommand{\argmax}{\mathop{\mathrm{argmax}}}
\newcommand{\argmin}{\mathop{\mathrm{argmin}}}

\newcommand{\Rmax}{R_{\max}}

\newcommand{\E}{\mathbb E}

\newcommand{\Qmax}{Q_{\max}}

\newcommand{\Acal}{\mathcal A}
\newcommand{\Dcal}{\mathcal D}
\newcommand{\Scal}{\mathcal S}

\usepackage{pifont}% http://ctan.org/pkg/pifont
\usepackage{xspace}
\def\algo{DMG\xspace}

\title{Doubly Mild Generalization for Offline \\ Reinforcement Learning}

\author{
  Yixiu Mao$^{1}$, Qi Wang$^{1}$, Yun Qu$^{1}$, Yuhang Jiang$^{1}$, Xiangyang Ji$^{1}$\\
  $^{1}$Department of Automation, Tsinghua University \\
  \texttt{myx21@mails.tsinghua.edu.cn,} \texttt{xyji@tsinghua.edu.cn}
}

\begin{document}

\maketitle

\begin{abstract}
Offline Reinforcement Learning (RL) suffers from the extrapolation error and value overestimation. From a generalization perspective, this issue can be attributed to the over-generalization of value functions or policies towards out-of-distribution (OOD) actions. Significant efforts have been devoted to mitigating such generalization, and recent in-sample learning approaches have further succeeded in entirely eschewing it. Nevertheless, we show that mild generalization beyond the dataset can be trusted and leveraged to improve performance under certain conditions. To appropriately exploit generalization in offline RL, we propose Doubly Mild Generalization (DMG), comprising (i) mild action generalization and (ii) mild generalization propagation. The former refers to selecting actions in a close neighborhood of the dataset to maximize the Q values. Even so, the potential erroneous generalization can still be propagated, accumulated, and exacerbated by bootstrapping. In light of this, the latter concept is introduced to mitigate the generalization propagation without impeding the propagation of RL learning signals. Theoretically, DMG guarantees better performance than the in-sample optimal policy in the oracle generalization scenario. Even under worst-case generalization, DMG can still control value overestimation at a certain level and lower bound the performance. Empirically, DMG achieves state-of-the-art performance across Gym-MuJoCo locomotion tasks and challenging AntMaze tasks. Moreover, benefiting from its flexibility in both generalization aspects, DMG enjoys a seamless transition from offline to online learning and attains strong online fine-tuning performance.
\end{abstract}
% TLDR: This work proposes Doubly Mild Generalization, comprising mild action generalization and mild generalization propagation, to appropriately exploit generalization in offline RL.

\section{Introduction}
Reinforcement learning~(RL) aims to solve sequential decision-making problems and has garnered significant attention in recent years~\citep{mnih2015human,silver2017mastering,vinyals2019grandmaster,schrittwieser2020mastering,degrave2022magnetic}. However, its practical applications encounter several challenges, such as risky exploration attempts~\cite{garcia2015comprehensive} and time-consuming data collection phases~\cite{kober2013reinforcement}.
Offline RL emerges as a promising paradigm to alleviate these challenges by learning without interaction with the environment~\cite{lange2012batch,levine2020offline}.
It eliminates the need for unsafe exploration and facilitates the utilization of pre-existing large-scale datasets~\citep{johnson2016mimic,maddern20171,qu2024hokoff}.

However, offline RL suffers from the out-of-distribution~(OOD) issue and extrapolation error~\citep{fujimoto2019off}. From a generalization perspective, this well-known challenge can be regarded as a consequence of the over-generalization of value functions or policies towards OOD actions~\cite{ma2023reining}.
Specifically, the potential value over-estimation at OOD actions caused by intricate generalization is often improperly captured by the max operation~\cite{van2016deep}. This over-estimation will propagate to values of in-distribution samples through Bellman backups and further spread to values of OOD ones via generalization.
In mitigating value overestimation caused by OOD actions, substantial efforts have been dedicated~\cite{fujimoto2019off,kumar2020conservative,kumar2019stabilizing,fujimoto2021minimalist} and recent advancements in in-sample learning have successfully formulated the Bellman target solely with the actions present in the dataset~\cite{kostrikov2022offline,xiao2023the,zhang2023insample,xu2023offline,garg2023extreme} and extracted policies by weighted behavior cloning~\cite{peng2019advantage,wang2020critic}. As a result, these algorithms completely eschew generalization and avoid the extrapolation error. Despite simplicity, this way can not take advantage of the generalization ability of neural networks, which could be beneficial for performance improvement. Until now, how to appropriately exploit generalization in offline RL remains a lasting issue.

This work demonstrates that mild generalization beyond the dataset can be trusted and leveraged to improve performance under certain conditions. For appropriate exploitation of mild generalization, we propose Doubly Mild Generalization~(\algo) for offline RL,  comprising (i) mild action generalization and (ii) mild generalization propagation. 
The former concept refers to choosing actions in the vicinity of the dataset to maximize the Q values. However, the mere utilization of mild action generalization still falls short in adequately circumventing potential erroneous generalization, which can be propagated, accumulated, and exacerbated through the process of bootstrapping. To address this, we propose a novel concept, mild generalization propagation, which involves reducing the generalization propagation while preserving the propagation of RL learning signals. 
Regarding \algo's implementation, this work presents a simple yet effective scheme. Specifically, we blend the mildly generalized max with the in-sample max in the Bellman target, where the former is achieved by actor-critic learning with regularization towards high-value in-sample actions, and the latter is accomplished using in-sample learning techniques such as expectile regression~\cite{kostrikov2022offline}.

We conduct a thorough theoretical analysis of our approach \algo in both oracle and worst-case generalization scenarios. Under oracle generalization, DMG guarantees better performance than the in-sample optimal policy in the dataset~\cite{kumar2019stabilizing,kostrikov2022offline}. Even under worst-case generalization, \algo can still upper bound the overestimation of value functions and guarantee to output a safe policy with a performance lower bound.
Empirically\footnote{Our code is available at \href{https://github.com/maoyixiu/DMG}{https://github.com/maoyixiu/DMG}.}, \algo achieves state-of-the-art performance on standard offline RL benchmarks~\citep{fu2020d4rl}, including Gym-MuJoCo locomotion tasks and challenging AntMaze tasks. Moreover, benefiting from its flexibility in both generalization aspects, \algo can seamlessly transition from offline to online learning and attain superior online fine-tuning performance.

\section{Preliminaries}
\paragraph{RL.}
The environment in RL is mostly characterized as a Markov decision process (MDP), which can be represented as a tuple $\mathcal{M}=(\mathcal{S},\mathcal{A},P,R,\gamma, d_0)$, comprising the state space $\mathcal{S}$, action space $\mathcal{A}$, transition dynamics $P: \Scal \times \Acal \to \Delta(\Scal)$, reward function $R: \Scal \times \Acal \to [0,\Rmax]$, discount factor $\gamma \in [0,1)$, and initial state distribution $d_0$~\cite{sutton2018reinforcement}.
The goal of RL is to find a policy $\pi: \Scal \to \Delta(\Acal)$ that can maximize the expected discounted return, denoted as $J(\pi)$:
\begin{equation}
J(\pi) = \mathbb{E}_{s_0\sim d_0, a_t\sim\pi(\cdot|s_t), s_{t+1}\sim P(\cdot|s_t,a_t)}\left[\sum_{t=0}^\infty\gamma^t R(s_t,a_t)\right].
\end{equation}

For any policy $\pi$, we define the value function as $V^\pi(s)=\mathbb{E}_\pi\left[\sum_{t=0}^{\infty}\gamma^t R(s_t,a_t) | s_0=s\right]$ and the state-action value function~($Q$-value function) as $Q^\pi(s,a)=\mathbb{E}_\pi\left[\sum_{t=0}^{\infty}\gamma^t R(s_t,a_t) | s_0=s,a_0=a\right]$. 

\paragraph{Offline RL.}
Distinguished from traditional online RL training, offline RL handles a static dataset of transitions $\mathcal{D}=\{(s_i,a_i,r_i,s_i')\}_{i=0}^{n-1}$ and seeks an optimal policy without any additional data collection~\cite{lange2012batch,levine2020offline}. We use $\hat\beta(a|s)$ to denote the empirical behavior policy observed in $\mathcal{D}$, which depicts the conditional distributions in the dataset~\cite{fujimoto2019off}. Ordinary approximate dynamic programming methods minimize temporal difference error, according to the following loss~\cite{sutton2018reinforcement}:
\begin{equation}
\label{eq:td loss}
L_{TD}(\theta) = \mathbb{E}_{(s, a, s')\sim \mathcal{D}}\left[(Q_\theta(s,a) - R(s,a) - \gamma \max_{a'}Q_{\theta'}(s',a'))^2\right],
\end{equation}
where $\pi_\phi$ is a parameterized policy,
$Q_\theta(s,a)$ is a parameterized $Q$ function, and $Q_{\theta'}(s,a)$ is a target $Q$ function whose parameters are updated via Polyak averaging~\cite{mnih2015human}.

\section{Doubly Mild Generalization for Offline RL}
This section discusses the strategy to appropriately exploit generalization in offline RL. In \cref{sec:Generalization issues in offline RL}, we introduce a formal perspective on how generalization impacts offline RL and discuss the issues of over-generalization and non-generalization. Subsequently, we propose the \algo concept, comprising mild action generalization and mild generalization propagation in \cref{sec:Doubly mild generalization}. Following this, we conduct a comprehensive analysis of \algo in both oracle generalization (\cref{sec:Oracle generalization}) and worst-case generalization scenarios (\cref{sec:Worst-case generalization}). Finally, we present the practical algorithm in \cref{sec:Practical algorithm}.

\subsection{Generalization Issues in Offline RL}
\label{sec:Generalization issues in offline RL}

Offline RL training typically involves a complex interaction between Bellman backup and generalization~\citep{ma2023reining}.
Offline RL algorithms vary in backup mechanisms to train the Q function. Here we denote a generic form of Bellman backup as $\mathcal{T}_{u}$, where $u$ is a distribution in the action space.
\begin{equation}
\mathcal{T}_{u} Q(s,a):=R(s,a)+\gamma \E_{s'\sim P(\cdot|s,a)}\left[\max_{a'\sim u(\cdot|s')} Q(s',a')\right]
\end{equation}

During offline training, this backup is exclusively executed on $(s,a)\in \Dcal$, and the values of $(s,a)\notin \Dcal$ are influenced solely via generalization. A crucial aspect is that $(s',a')$ in the Bellman target can be absent from the dataset $\Dcal$, depending on the choice of $u$.
As a result, Bellman backup and generalization exhibit an intricate interaction: the backups on $(s,a)\in \Dcal$ impact the values of $(s,a)\notin \Dcal$ via generalization; the values of $(s,a)\notin \Dcal$ participates in the computation of Bellman target, thereby affecting the values of $(s,a)\in \Dcal$.

This interaction poses a key challenge in offline RL, value overestimation. The potential overestimation of values of $(s,a)\notin \Dcal$, induced by intricate generalization, tends to be improperly captured by the max operation, a phenomenon known as maximization bias~\cite{van2016deep}. This overestimation propagates to values of $(s,a)\in \Dcal$ through backups and further extends to values of $(s,a)\notin \Dcal$ via generalization. This cyclic process consistently amplifies value overestimation, potentially resulting in value divergence. The crux of this detrimental process can be summarized as \textbf{over-generalization}.

To address value overestimation, recent advancements in the field have introduced a paradigm known as in-sample learning, which formulates the Bellman target solely with the actions present in the dataset~\cite{kostrikov2022offline,xiao2023the,zhang2023insample,xu2023offline,garg2023extreme}. Its effect is equivalent to choosing $u$ in $\mathcal{T}_{u}$ to be exactly $\hat \beta$, i.e., the empirical behavior policy observed in the dataset. Following in-sample value learning, policies are extracted from the learned Q functions using weighted behavior cloning~\cite{peng2019advantage,chen2020bail,nair2020awac}. By entirely eschewing generalization in offline RL training, they effectively avoid the extrapolation error~\cite{fujimoto2019off}, a strategy we term \textbf{non-generalization}. 
However, the ability to generalize is a critical factor contributing to the extensive utilization of neural networks~\cite{lecun2015deep}. In this sense, in-sample learning methods seem too conservative without utilizing generalization, particularly when the offline datasets do not cover the optimal actions in large or continuous spaces.

\subsection{Doubly Mild Generalization}
\label{sec:Doubly mild generalization}

The following focuses on the appropriate exploitation of generalization in offline RL.

We start by analyzing the generalization effect under the generic backup operator $\mathcal{T}_{u}$. We consider a straightforward scenario, where $Q_\theta$ is updated to $Q_{\theta'}$ by one gradient step on a single $(s,a)\in \Dcal$ with learning rate $\alpha$. We characterize the resulting generalization effect on any $(s,\tilde a)\notin \Dcal$\footnote{The interplay between backup and generalization does not involve states out of the dataset (Bellman target does not contain OOD states), hence we do not consider $(\tilde s,\tilde a)\notin \Dcal$, though the analysis of $Q(\tilde s,\tilde a)$ is similar.} as follows.

\begin{theorem}[Informal]
\label{prop:gen}
Under certain continuity conditions, the following equation holds when the learning rate $\alpha$ is sufficiently small and $\tilde a$ is sufficiently close to $a$:
\begin{equation}
\label{eq:gen}
Q_{\theta'}(s,\tilde a) = Q_{\theta}(s,\tilde a) + C_1 \left(\mathcal{T}_{u}Q_{\theta}(s,\tilde a) -Q_{\theta}(s,\tilde a) + C_2\|\tilde a-a\|\right) + \mathcal{O}\left(\|\theta'-\theta\|^2\right)
\end{equation}
where $C_1 \in [0,1]$ and $C_2$ is a bounded constant.
\end{theorem}
The formal theorem and all proofs are deferred to \cref{app_sec:proofs}.

Note that Eq.~\eqref{eq:gen} is the update of the parametric Q function ($Q_{\theta} \rightarrow Q_{\theta'}$) at state-action pairs $(s,\tilde a) \notin \mathcal D$, which is exclusively caused by generalization. If $\tilde a$ is within a close neighborhood of $a$, then $C_2\|\tilde a-a\|$ is small. Moreover, as $C_1 \in [0,1]$, Eq.~\eqref{eq:gen} approximates an update towards the true objective $\mathcal T_u Q_\theta(s,\tilde a)$, as if $Q_\theta(s,\tilde a)$ is updated by a true gradient step at $(s,\tilde a) \notin \mathcal D$.
Therefore, \cref{prop:gen} shows that, under certain continuity conditions, Q functions can generalize well and approximate true updates in a close neighborhood of samples in the dataset.
This implies that mild generalizations beyond the dataset can be leveraged to potentially pursue better performance. Inspired by \cref{prop:gen}, we define a mildly generalized policy $\tilde \beta$ as follows.

\begin{definition}[Mildly generalized policy]
\label{def:mg}
Policy $\tilde \beta$ is termed a mildly generalized policy if it satisfies
\begin{equation}
\mathrm{supp}(\hat \beta(\cdot|s)) \subseteq \mathrm{supp}(\tilde \beta(\cdot|s)), ~~\text{and}~~ \max_{a_1 \sim \tilde \beta(\cdot|s)} \min_{a_2 \sim \hat \beta(\cdot|s)} \|a_1-a_2\| \leq \epsilon_a,
\end{equation}
where $\hat \beta$ is the empirical behavior policy observed in the offline dataset.
\end{definition}
It means that $\tilde \beta$ has a wider support than $\hat \beta$ (the dataset), and for any $a_1 \sim \tilde \beta(\cdot|s)$, we can find $a_2 \sim \hat \beta(\cdot|s)$ (in dataset) such that $\|a_1-a_2\| \leq \epsilon_a$. In other words, the generalization of $\tilde \beta$ beyond the dataset is bounded by $\epsilon_a$ when measured in the action space distance. According to \cref{prop:gen}, there is a high chance that $Q_{\theta}$ can generalize well in this mild generalization area $\tilde \beta(a|s)>0$.

However, even in this mild generalization area, it is inevitable that the learned value function will incur some degree of generalization error. The possible erroneous generalization can still be propagated and exacerbated by value bootstrapping as discussed in \cref{sec:Generalization issues in offline RL}. To this end, we introduce an additional level of mild generalization, termed mild generalization propagation, and propose a novel Doubly Mildly Generalization (\algo) operator as follows.

\begin{definition}
\label{def:algo operator}
The Doubly Mild Generalization (\algo) operator is defined as
\begin{equation}
\mathcal{T}_{\mathrm{\algo}} Q(s,a):=R(s,a)+\gamma \E_{s'\sim P(\cdot|s,a)}\left[\lambda \max_{a'\sim \tilde \beta(\cdot|s')} Q(s',a') + (1-\lambda)\max_{a'\sim \hat \beta(\cdot|s')} Q(s',a')\right]
\end{equation}
where $\hat \beta$ is the empirical behavior policy in the dataset and $\tilde \beta$ is a mildly generalized policy.
\end{definition}

Note that in typical offline RL algorithms, extrapolation error and value overestimation caused by erroneous generalization are propagated through bootstrapping, and the discount factor of this process is $\gamma$. \algo reduces this discount factor to $\lambda \gamma$, mitigating the amplification of value overestimation. On the other hand, in contrast to in-sample methods, \algo allows mild generalization, utilizing the generalization ability of neural networks to seek better performance, as \cref{prop:gen} suggests that value functions are highly likely to generalize well in the mild generalization area.

To summarize, the generalization of \algo is mild in two aspects: (i) \textbf{mild action generalization}: based on the mildly generalized policy $\tilde \beta$, which generalizes beyond $\hat \beta$, \algo selects actions in a close neighborhood of the dataset to maximize the Q values in the first part of the Bellman target; and (ii) \textbf{mild generalization propagation}: \algo mitigates the generalization propagation without hindering the propagation of RL learning signals by blending the mildly generalized max with the in-sample max in the Bellman target. This reduces the discount factor through which generalization propagates, mitigating the amplification of value overestimation caused by bootstrapping.

To support the above claims, we provide a comprehensive analysis of \algo in both oracle and worst-case generalization scenarios, with particular emphasis on value estimation and performance.

\subsection{Oracle Generalization}
\label{sec:Oracle generalization}
This section conducts analyses under the assumption that the learned value functions can achieve oracle generalization in the mild generalization area $\tilde \beta(a|s)>0$, formally defined as follows. 
\begin{assumption}[Oracle generalization]
\label{ass:Oracle generalization}
The generalization of learned Q functions in the mild generalization area $\tilde \beta(a|s)>0$ reflects the true value updates according to $\mathcal{T}_{\mathrm{\algo}}$.
\end{assumption}
The mild generalization area $\tilde \beta(a|s)>0$ may contain some points outside the offline dataset, and $\mathcal{T}_{\mathrm{\algo}}$ might query Q values of such points. This assumption assumes that the generalization at such points reflects the true value updates according to $\mathcal{T}_{\mathrm{\algo}}$.
The rationale for such an assumption comes from \cref{prop:gen}, which characterizes the generalization effect of value functions in the mild generalization area.
Now we analyze the dynamic programming properties of the operators $\mathcal{T}_{\mathrm{\algo}}$ and $\mathcal{T}_{\mathrm{In}}$, where $\mathcal{T}_{\mathrm{In}}$ is the in-sample Q learning operator~\cite{kostrikov2022offline,xu2023offline,garg2023extreme} defined as follows.

\begin{definition}
\label{def:in operator}
The In-sample Q Learning operator~\cite{kostrikov2022offline} is defined as
\begin{equation}
\mathcal{T}_{\mathrm{In}} Q(s,a):=R(s,a)+\gamma \E_{s'\sim P(\cdot|s,a)}\left[\max_{a'\sim \hat \beta(\cdot|s')} Q(s',a')\right]
\end{equation}
where $\hat \beta$ is the empirical behavior policy in the dataset.
\end{definition}

\begin{lemma}
\label{lem:in Contraction}
$\mathcal{T}_{\mathrm{In}}$ is a $\gamma$-contraction operator in the in-sample area $\hat \beta(a|s)>0$ under the $\mathcal{L}_\infty$ norm.
\end{lemma}

Following \cref{lem:in Contraction}, we denote the fixed point of $\mathcal{T}_{\mathrm{In}}$ as $Q^*_{\mathrm{In}}$, and its induced policy as $\pi^*_{\mathrm{In}}$.
Here $Q^*_{\mathrm{In}}$ is known as the in-sample optimal value function~\cite{kostrikov2022offline}, which is the value function of the in-sample optimal policy $\pi^*_{\mathrm{In}}$. We refer readers to \cite{kostrikov2022offline,kumar2019stabilizing,xu2023offline} for more discussions on the in-sample optimality.

Now we present the theoretical properties of \algo for comparison.

\begin{theorem}[Contraction]
\label{prop:Contraction}
Under \cref{ass:Oracle generalization}, $\mathcal{T}_{\mathrm{\algo}}$ is a $\gamma$-contraction operator in the mild generalization area $\tilde \beta(a|s)>0$ under the $\mathcal{L}_\infty$ norm.
Therefore, by repeatedly applying $\mathcal{T}_{\mathrm{\algo}}$, any initial Q function can converge to the unique fixed point $Q^*_{\mathrm{\algo}}$.
\end{theorem}

We denote the induced policy of $Q^*_{\mathrm{\algo}}$ as $\pi^*_{\mathrm{\algo}}$, whose performance is guaranteed as follows.

\begin{theorem}[Performance]
\label{thm:Performance}
Under \cref{ass:Oracle generalization}, the value functions of $\pi^*_{\mathrm{\algo}}$ and $\pi^*_{\mathrm{In}}$ satisfy:
\begin{equation}
V^{\pi^*_{\mathrm{\algo}}}(s) \geq V^{\pi^*_{\mathrm{In}}}(s), \quad \forall s \in \Dcal.
\end{equation}
\end{theorem}

\cref{thm:Performance} indicates that the policy learned by \algo can achieve better performance than the in-sample optimal policy under the oracle generalization condition.

\subsection{Worst-case Generalization}
\label{sec:Worst-case generalization}
This section turns to the analyses in the worst-case generalization scenario, where the learned value functions may exhibit poor generalization in the mild generalization area $\tilde \beta(a|s)>0$. In other words, this section considers that $\mathcal{T}_{\mathrm{\algo}}$ is only defined in the in-sample area $\hat \beta(a|s)>0$ and the learned value functions may have any generalization error at other state-action pairs. In this case, we use the notation $\hat{\mathcal{T}}_{\mathrm{\algo}}$ to tell the difference. 

We make continuity assumptions about the learned Q function and the transition dynamics.

\begin{assumption}[Lipschitz $Q$]
\label{ass:lQ}
The learned Q function is $K_Q$-Lipschitz. $\forall s \sim \Dcal$, $\forall a_1, a_2 \sim \Acal$, $|Q(s,a_1)-Q(s,a_2)| \leq K_Q \|a_1-a_2\|$
\end{assumption}

\begin{assumption}[Lipschitz $P$]
\label{ass:lP}
The transition dynamics $P$ is $K_P$-Lipschitz. $\forall s,s' \sim \Scal$, $\forall a_1,a_2 \sim \Acal$, $|P(s'|s,a_1)-P(s'|s,a_2)| \leq K_P \|a_1-a_2\|$
\end{assumption}

For \cref{ass:lQ}, a continuous learned Q function is particularly necessary for analyzing value function generalization and can be relatively easily satisfied using neural networks or linear models~\cite{gouk2021regularisation}. \cref{ass:lP} is also a common assumption in theoretical studies of RL~\cite{dufour2013finite,xiong2022deterministic,ran2023policy}.

Now we consider the iteration starting from arbitrary function $Q^0$: $\hat Q^k_{\mathrm{\algo}}=\hat{\mathcal{T}}_{\mathrm{\algo}} \hat Q^{k-1}_{\mathrm{\algo}}$ and $Q^k_{\mathrm{In}}=\mathcal{T}_{\mathrm{In}} Q^{k-1}_{\mathrm{In}}$, $\forall k \in \mathbb{Z}^+$. The possible value of $\hat Q^k_{\mathrm{\algo}}$ is bounded by the following results.
\begin{theorem}[Limited overestimation]
\label{prop:Limited over-estimation}
Under \cref{ass:lQ}, the learned Q function of \algo by iterating $\hat{\mathcal{T}}_{\mathrm{\algo}}$ satisfies the following inequality
\begin{equation}
\label{eq:over-e}
Q^k_{\mathrm{In}}(s,a) \leq \hat Q^k_{\mathrm{\algo}}(s,a) \leq Q^k_{\mathrm{In}}(s,a) + \frac{\lambda\epsilon_a K_Q \gamma}{1-\gamma}(1-\gamma^k), ~\forall s,a \sim \Dcal, ~\forall k \in \mathbb{Z}^+.
\end{equation}
\end{theorem}
Since in-sample training eliminates the extrapolation error~\cite{kostrikov2022offline,zhang2023insample}, $Q^k_{\mathrm{In}}$ can be considered a relatively accurate estimate~\cite{kostrikov2022offline}. Therefore, \cref{prop:Limited over-estimation} suggests that \algo exhibits limited value overestimation under the worst-case generalization scenario. Moreover, the bound becomes tighter as $\epsilon_a$ decreases (milder action generalization) and $\lambda$ decreases (milder generalization propagation). This is consistent with our intuitions in \cref{sec:Doubly mild generalization}.

Finally, we show in \cref{prop:Performance lower bound} that even under worst-case generalization, \algo guarantees to output a safe policy with a performance lower bound.

\begin{theorem}[Performance lower bound]
\label{prop:Performance lower bound}
Let $\hat\pi_{\mathrm{\algo}}$ be the learned policy of \algo by iterating $\hat{\mathcal{T}}_{\mathrm{\algo}}$, $\pi^*$ be the optimal policy, and $\epsilon_{\Dcal}$ be the inherent performance gap of the in-sample optimal policy $\epsilon_{\Dcal}:=J(\pi^*)-J(\pi^*_{\mathrm{In}})$. Under Assumptions \ref{ass:lQ} and \ref{ass:lP}, for sufficiently small $\epsilon_a$, we have
\begin{equation}
J(\hat{\pi}_{\mathrm{\algo}}) \geq J(\pi^*) - \frac{CK_P\Rmax}{1-\gamma}\epsilon_a - \epsilon_{\Dcal}.
\end{equation}
where $C$ is a positive constant.
\end{theorem}

\subsection{Practical Algorithm}
\label{sec:Practical algorithm}
This section puts \algo into implementation and presents a simple yet effective practical algorithm. The algorithm comprises the following networks: policy $\pi_\phi$, target policy $\pi_{\phi'}$, $Q$ network $Q_\theta$, target $Q$ network $Q_{\theta'}$, and $V$ network $V_\psi$.
\paragraph{Policy learning.}
Practically, we expect \algo to exhibit a tendency towards mild generalization around good actions in the dataset. To this end, we first consider reshaping the empirical behavior policy $\hat \beta$ to be skewed towards actions with high advantage values $\hat\beta^*(a|s) \propto \hat\beta(a|s)\exp(A(s,a)) $. Then we enforce the proximity between the trained policy and the reshaped behavior policy to constrain the generalization area. We define the generalization set $\Pi_G$ as follows.
\begin{equation}
\Pi_G = \{ \pi ~|~ \mathrm{KL}(\hat\beta^*(\cdot|s) \| \pi(\cdot|s)) \leq \epsilon \}
\end{equation}
Note that forward KL allows $\pi$ to select actions outside the support of $\hat\beta^*$, enabling $\Pi_G$ to generalize beyond the actions in the dataset. With $\Pi_G$ defined, the next step is to compute the maximal $Q$ within $\Pi_G$. To accomplish this, we adopt Actor-Critic style training~\cite{sutton2018reinforcement} for this part.
\begin{equation}
\max_\phi \mathbb{E}_{s\sim \mathcal{D}, a\sim \pi_\phi(\cdot|s)} Q_\theta(s,a), \quad s.t. ~  \pi_\phi \in \Pi_G
\end{equation}
By treating the constraint term as a penalty, we maximize the following objective.
\begin{equation}
\label{eq:KLreg}
\max_\phi \mathbb{E}_{s\sim \mathcal{D}, a\sim \pi_\phi(\cdot|s)} Q_\theta(s,a) - \nu \mathbb{E}_{s\sim \mathcal{D}} \mathrm{KL}(\hat\beta^*(\cdot|s) \| \pi_\phi(\cdot|s))
\end{equation}
Through straightforward derivations, Eq.~\eqref{eq:KLreg} is equivalent to the following policy training objective.
\begin{equation}
\label{eq:pi}
J_{\pi}(\phi) = \mathbb{E}_{s\sim \mathcal{D}, a\sim \pi_\phi(\cdot|s)} Q_\theta(s,a) - \nu \mathbb{E}_{(s,a)\sim \mathcal{D}} \left[\exp (\alpha(Q_{\theta'}(s,a) - V_\psi(s))) \log \pi_\phi(a|s) \right]
\end{equation}
where $\alpha$ is an inverse temperature and $Q_{\theta'}(s,a) - V_\psi(s)$ computes the advantage function $A(s,a)$.

\begin{wrapfigure}{R}{0.466\textwidth}
\vspace{-0.7cm}
\begin{minipage}{0.466\textwidth}
\begin{algorithm}[H]
\caption{\algo}
\label{alg}
\begin{algorithmic}[1] %[1] enables line numbers
\STATE Initialize $\pi_\phi$, $\pi_{\phi'}$, $Q_\theta$, $Q_{\theta'}$, and $V_\psi$.
\FOR{each gradient step}
\STATE Update $\psi$ by minimizing Eq.~(\ref{eq:V})
\STATE Update $\theta$ by minimizing Eq.~(\ref{eq:Q})
\STATE Update $\phi$ by maximizing Eq.~(\ref{eq:pi})
\STATE Update target networks: $\theta' \leftarrow (1-\xi){\theta'} + \xi\theta$, $\phi' \leftarrow (1-\xi){\phi'} + \xi\phi$
\ENDFOR
\end{algorithmic}
\end{algorithm}
\end{minipage}
\vspace{-0.1cm}
\end{wrapfigure}

\paragraph{Value learning.}
Now we turn to the implementation of the $\mathcal{T}_{\mathrm{\algo}}$ operator for training value functions. By introducing the aforementioned policy, we can substitute $\max_{a\sim \tilde \beta}$ in $\mathcal{T}_{\mathrm{\algo}}$ with $\mathbb{E}_{a\sim \pi}$. Regarding $\max_{a\sim \hat \beta}$ in $\mathcal{T}_{\mathrm{\algo}}$, any in-sample learning techniques can be employed to compute the in-sample maximum~\cite{kostrikov2022offline,xu2023offline,xiao2023the,garg2023extreme}. In particular, based on IQL~\cite{kostrikov2022offline}, we perform expectile regression.
\begin{equation}
\label{eq:V}
L_{V}(\psi) = \underset{(s,a)\sim \mathcal{D}}{\mathbb{E}} \left[L^\tau_2 \left(Q_{\theta'}(s,a) - V_\psi(s) \right) \right]
\end{equation}
where $L_2^\tau(u) = |\tau-\mathbbm{1}(u < 0)|u^2$ and $\tau \in (0,1)$. For $\tau \approx 1$, $V_\psi$ can capture the in-sample maximal $Q$~\cite{kostrikov2022offline}. Finally, we have the following value training loss.
\begin{equation}
\label{eq:Q}
L_{Q}(\theta) = \mathbb{E}_{(s,a,s')\sim \mathcal{D}}\left[\left(Q_\theta(s,a) - R(s,a) - \gamma \lambda \mathbb{E}_{a'\sim \pi_{\phi'}}Q_{\theta'}(s',a') - \gamma(1-\lambda)V_\psi(s')\right)^2\right]
\end{equation}

\paragraph{Overall algorithm.}
Integrating all components, we present our practical algorithm in \cref{alg}.

\section{Discussions and Related Work}
\paragraph{Summary of offline RL work from a generalization perspective.}
As analyzed above, \algo is featured in both mild action generalization and mild generalization propagation. Within the actor-critic framework upon which most offline RL algorithms are built, these two aspects correspond to the policy and value training phases, respectively. Action generalization concerns whether the policy training intentionally selects actions beyond the dataset to maximize Q values, while generalization propagation involves whether value training propagates generalization through bootstrapping. \cref{tab:summary} presents a clear comparison of offline RL works in this generalization view. The table shows one representative method of each category and we elaborate on others as follows.

\begin{table}[htbp]
\caption{Comparison of offline RL work from the generalization perspective.}
\vspace{-2mm}
\label{tab:summary}
\begin{center}
\begin{tabular}{lccccc}
\toprule
 &IQL&AWAC&TD3BC&TD3&\algo~(Ours)      \\ 
\midrule
Action generalization&\textit{none}&\textit{none}&\textit{mild}&\textit{full}& \textit{mild}  \\
\midrule
Generalization propagation&\textit{none}&\textit{full}&\textit{full}&\textit{full}& \textit{mild} \\
\bottomrule
\end{tabular}
\end{center}
\vspace{-2mm}
\end{table}

Concerning policy learning, AWR~\cite{peng2019advantage}, AWAC~\cite{nair2020awac}, CRR~\cite{wang2020critic}, $10\%$ BC~\cite{chen2021decision}, IQL~\cite{kostrikov2022offline}, and other works such as \cite{wang2018exponentially,chen2020bail,Siegel2020Keep,garg2023extreme,xu2023offline} extract policies through weighted or filtered behavior cloning, thereby lacking intentional action generalization to maximize Q values beyond the dataset. Typical policy-regularized offline RL methods like TD3BC~\cite{fujimoto2021minimalist}, BRAC~\cite{wu2019behavior}, BEAR~\cite{kumar2019stabilizing}, SPOT~\cite{wu2022supported}, and others such as \cite{wang2023diffusion,ran2023policy,tarasov2024revisiting} introduce regularization terms to Q maximization objectives to regularize the trained policy towards the behavior policy and allows mild action generalization. Online RL algorithms like TD3~\cite{fujimoto2018addressing} and SAC~\cite{haarnoja2018soft} have no constraints and maximize Q values in the entire action space, corresponding to full action generalization.
Regarding value training, in-sample learning methods including OneStep RL~\cite{brandfonbrener2021offline}, IQL~\cite{kostrikov2022offline}, InAC~\cite{xiao2023the}, IAC~\cite{zhang2023insample}, $\mathcal{X}$QL~\cite{garg2023extreme}, and SQL~\cite{xu2023offline} completely avoid generalization propagation and accumulation via bootstrapping, whereas typical offline and online RL approaches allow full generalization propagation through bootstrapping. In the proposed approach \algo, generalization is mild in both aspects.

Recently, \citet{ma2023reining} have also drawn attention to generalization in offline RL and the issue of over-generalization. They mitigate over-generalization from a representation perspective, differentiating between the representations of in-sample and OOD state-action pairs. \citet{lyu2022mildly} argue that conventional value penalization like CQL~\cite{kumar2020conservative} tends to harm the generalization of value functions and hinder performance improvement. They propose mild value penalization to mitigate the detrimental effects of value penalization on generalization.

\paragraph{Connection to heuristic blending approaches.}
Our approach also relates to the framework of blending heuristics into bootstrapping~\cite{cheng2021heuristic,wilcox2022monte,sutton2016emphatic,imani2018off,wright2013exploiting,geng2024improving}. In offline RL, HUBL~\cite{geng2024improving} blends Monte-Carlo returns into bootstrapping and acts as a data relabeling step, which reduces the degree of bootstrapping and thereby increases its performance. In contrast, \algo blends the in-sample maximal values into the bootstrapping operator. \algo does not reduce the discount for RL learning but reduces the discount for generalization propagation. 

For extended discussions on related work, please refer to \cref{app_sec:related work}.

\section{Experiments}
In this section, we conduct several experiments to justify the validity of the proposed method \algo. 
Experimental details and extended results are provided in \cref{app_sec:experimental_details,app_sec:experimental_results}, respectively.

\begin{table}[t]
  \caption{Averaged normalized scores on Gym locomotion and Antmaze tasks over five random seeds. 
  m = medium, m-r = medium-replay, m-e = medium-expert, e = expert, r = random; u = umaze, u-d = umaze-diverse, m-p = medium-play, m-d = medium-diverse, l-p= large-play, l-d = large-diverse.
  }
  \vspace{-1mm}
  \label{tab:d4rl}
  \small
    \footnotesize
  \centering
  \setlength{\tabcolsep}{3.7pt}
  \begin{adjustbox}{max width=390pt}
  \begin{tabular}{@{}l rrrrrrrrr r@{}}
    \toprule
    Dataset-v2 & BC & BCQ & BEAR & DT & AWAC & OneStep & TD3BC & CQL & IQL & \algo~(Ours) \\ \midrule
    halfcheetah-m & 42.0 & 46.6 & 43.0 & 42.6 & 47.9 & 50.4 & 48.3 & 47.0 & 47.4 & \textbf{54.9$\pm$0.2} \\ 
    hopper-m & 56.2 & 59.4 & 51.8 & 67.6 & 59.8 & 87.5 & 59.3 & 53.0 & 66.2 & \textbf{100.6$\pm$1.9} \\ 
    walker2d-m & 71.0 & 71.8 & -0.2 & 74.0 & 83.1 & 84.8 & 83.7 & 73.3 & 78.3 & \textbf{92.4$\pm$2.7} \\ 
    halfcheetah-m-r & 36.4 & 42.2 & 36.3 & 36.6 & 44.8 & 42.7 & 44.6 & 45.5 & 44.2 & \textbf{51.4$\pm$0.3} \\ 
    hopper-m-r & 21.8 & 60.9 & 52.2 & 82.7 & 69.8 & 98.5 & 60.9 & 88.7 & 94.7 & \textbf{101.9$\pm$1.4} \\ 
    walker2d-m-r & 24.9 & 57.0 & 7.0 & 66.6 & 78.1 & 61.7 & 81.8 & 81.8 & 73.8 & \textbf{89.7$\pm$5.0} \\ 
    halfcheetah-m-e & 59.6 & \textbf{95.4} & 46.0 & 86.8 & 64.9 & 75.1 & 90.7 & 75.6 & 86.7 & 91.1$\pm$4.2 \\ 
    hopper-m-e & 51.7 & 106.9 & 50.6 & 107.6 & 100.1 & 108.6 & 98.0 & 105.6 & 91.5 & \textbf{110.4$\pm$3.4} \\ 
    walker2d-m-e & 101.2 & 107.7 & 22.1 & 108.1 & 110.0 & 111.3 & 110.1 & 107.9 & 109.6 & \textbf{114.4$\pm$0.7} \\ 
    halfcheetah-e & 92.9 & 89.9 & 92.7 & 87.7 & 81.7 & 88.2 & \textbf{96.7} & \textbf{96.3} & \textbf{95.0} & \textbf{95.9$\pm$0.3} \\ 
    hopper-e & \textbf{110.9} & 109.0 & 54.6 & 94.2 & \textbf{109.5} & 106.9 & 107.8 & 96.5 & \textbf{109.4} & \textbf{111.5$\pm$2.2} \\ 
    walker2d-e & 107.7 & 106.3 & 106.6 & 108.3 & 110.1 & 110.7 & 110.2 & 108.5 & 109.9 & \textbf{114.7$\pm$0.4} \\ 
    halfcheetah-r & 2.6 & 2.2 & 2.3 & 2.2 & 6.1 & 2.3 & 11.0 & 17.5 & 13.1 & \textbf{28.8$\pm$1.3} \\ 
    hopper-r & 4.1 & 7.8 & 3.9 & 5.4 & 9.2 & 5.6 & 8.5 & 7.9 & 7.9 & \textbf{20.4$\pm$10.4} \\ 
    walker2d-r & 1.2 & 4.9 & \textbf{12.8} & 2.2 & 0.2 & 6.9 & 1.6 & 5.1 & 5.4 & 4.8$\pm$2.2 \\ \midrule
    locomotion total & 784.2 & 968.0 & 581.7 & 972.6 & 975.3 & 1041.2 & 1013.2 & 1010.2 & 1033.1 & \textbf{1182.8} \\ 
    \midrule
    antmaze-u & 66.8 & 78.9 & 73.0 & 54.2 & 80.0 & 54.0 & 73.0 & 82.6 & 89.6 & \textbf{92.4$\pm$1.8} \\ 
    antmaze-u-d & 56.8 & 55.0 & 61.0 & 41.2 & 52.0 & 57.8 & 47.0 & 10.2 & 65.6 & \textbf{75.4$\pm$8.1} \\ 
    antmaze-m-p & 0.0 & 0.0 & 0.0 & 0.0 & 0.0 & 0.0 & 0.0 & 59.0 & 76.4 & \textbf{80.2$\pm$5.1} \\ 
    antmaze-m-d & 0.0 & 0.0 & 8.0 & 0.0 & 0.2 & 0.6 & 0.2 & 46.6 & 72.8 & \textbf{77.2$\pm$6.1} \\ 
    antmaze-l-p & 0.0 & 6.7 & 0.0 & 0.0 & 0.0 & 0.0 & 0.0 & 16.4 & 42.0 & \textbf{55.4$\pm$6.2} \\ 
    antmaze-l-d & 0.0 & 2.2 & 0.0 & 0.0 & 0.0 & 0.2 & 0.0 & 3.2 & 46.0 & \textbf{58.8$\pm$4.5} \\ \midrule
    antmaze total & 123.6 & 142.8 & 142.0 & 95.4 & 132.2 & 112.6 & 120.2 & 218.0 & 392.4 & \textbf{439.4} \\ 
    \bottomrule
  \end{tabular}
  \end{adjustbox}
  \vspace{-1.5mm}
\end{table}

\subsection{Main Results on Offline RL Benchmarks}

\paragraph{Tasks.} We evaluate the proposed approach on Gym-MuJoCo locomotion domains and challenging AntMaze domains in D4RL~\cite{fu2020d4rl}. The latter involves sparse-reward tasks and necessitates “stitching” fragments of suboptimal trajectories traveling undirectedly to find a path to the goal of the maze.

\paragraph{Baselines.} Our offline RL baselines include both typical bootstrapping methods and in-sample learning approaches. For the former, we compare to BCQ~\cite{fujimoto2019off}, BEAR~\cite{kumar2019stabilizing}, AWAC~\cite{nair2020awac}, TD3BC~\cite{fujimoto2021minimalist}, and CQL~\cite{kumar2020conservative}. For the latter, we compare to BC~\cite{pomerleau1988alvinn}, OneStep RL~\cite{brandfonbrener2021offline}, IQL~\cite{kostrikov2022offline}, $\mathcal{X}$QL~\cite{garg2023extreme}, and SQL~\cite{xu2023offline}. We also include the sequence-modeling method Decision Transformer~(DT)~\cite{chen2021decision}.

\paragraph{Comparison with baselines.} Aggregated results are displayed in \cref{tab:d4rl}. On the Gym locomotion tasks, \algo outperforms prior methods on most tasks and achieves the highest total score. On the much more challenging AntMaze tasks, \algo outperforms all the baselines by a large margin, especially in the most difficult large mazes. For detailed learning curves, please refer to \cref{app_sec:offline curves}.
According to \citep{patterson2023empirical}, we also report the results of \algo over more random seeds in \cref{app_sec:10 seeds}.

\paragraph{Runtime.} 
We test the runtime of \algo and other baselines on a GeForce RTX 3090. As shown in \cref{app_sec:computational cost}, the runtime of \algo is comparable to that of the fastest offline RL algorithm TD3BC.

\subsection{Performance Improvement over In-sample Learning Approaches}

\begin{wraptable}{r}{0.53\textwidth}
\vspace{-6.5mm}
\begin{center}
\begin{small}
\setlength{\tabcolsep}{7.0pt}
\caption{\algo combined with various in-sample approaches, showing averaged scores over 5 seeds.}
\label{tab:in-sample}
\vspace{2mm}
\begin{tabular}{lcc}
\toprule
Dataset-v2 & $\mathcal{X}$QL (+\algo) & SQL(+\algo) \\ \midrule
halfcheetah-m & 47.7 $\rightarrow$ \textbf{55.3} & 48.3 $\rightarrow$ \textbf{54.5} \\
hopper-m & 71.1 $\rightarrow$ \textbf{90.1} & 75.5 $\rightarrow$ \textbf{97.7} \\
walker2d-m & 81.5 $\rightarrow$ \textbf{88.7} & 84.2 $\rightarrow$ \textbf{89.8} \\
halfcheetah-m-r & 44.8 $\rightarrow$ \textbf{51.1} & 44.8 $\rightarrow$ \textbf{51.8} \\
hopper-m-r & 97.3 $\rightarrow$ \textbf{102.5} & \textbf{101.7} $\rightarrow$ \textbf{101.8} \\
walker2d-m-r & 75.9 $\rightarrow$ \textbf{90.0} & 77.2 $\rightarrow$ \textbf{95.2} \\
halfcheetah-m-e & 89.8 $\rightarrow$ \textbf{92.5} & \textbf{94.0} $\rightarrow$ \textbf{93.5} \\
hopper-m-e & 107.1 $\rightarrow$ \textbf{111.1} & \textbf{111.8} $\rightarrow$ \textbf{110.4} \\
walker2d-m-e & 110.1 $\rightarrow$ \textbf{111.3} & \textbf{110.0} $\rightarrow$ \textbf{109.6} \\
\midrule
total & 725.3 $\rightarrow$ \textbf{792.7} & 747.5 $\rightarrow$ \textbf{804.2} \\
\bottomrule
\end{tabular}
\end{small}
\end{center}
\vspace{-3mm}
\end{wraptable}

\algo can be combined with various in-sample learning approaches. Besides IQL~\cite{kostrikov2022offline}, we also apply \algo to two recent state-of-the-art in-sample algorithms, $\mathcal{X}$QL~\cite{garg2023extreme} and SQL~\cite{xu2023offline}. As shown in \cref{tab:in-sample} (and \cref{tab:d4rl}), \algo consistently and substantially improves upon these in-sample methods, particularly on sub-optimal datasets where generalization plays a crucial role in the pursuit of a better policy. This provides compelling empirical evidence that the performance of in-sample methods is largely confined by eschewing generalization beyond the dataset, while \algo effectively exploits generalization, achieving significantly improved performance across tasks.

\subsection{Ablation Study for Performance and Value Estimation}

\begin{figure}[t]
    \centering
    \includegraphics[width=\textwidth]{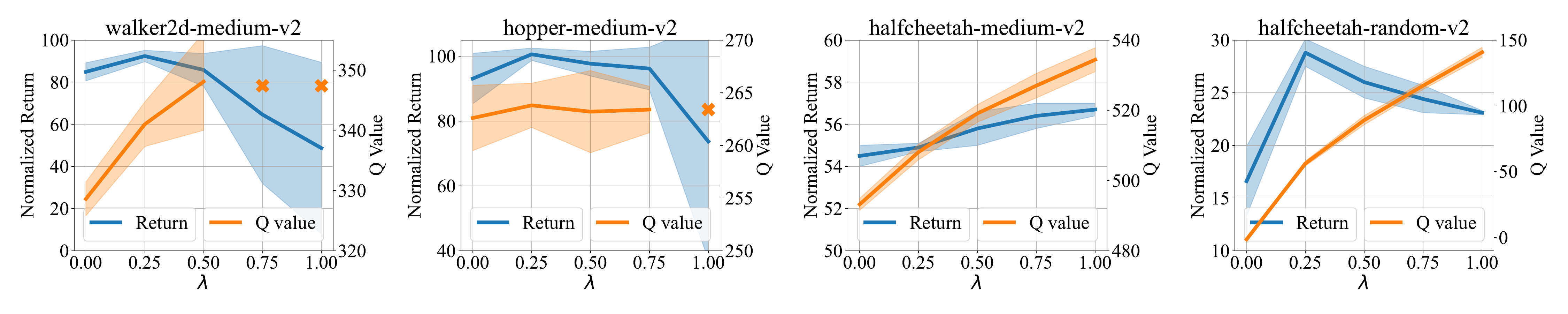}
    \vspace{-6mm}
    \caption{
    Performance and Q values of \algo with varying mixture coefficient $\lambda$ over 5 random seeds. The crosses $\times$ mean that the value functions diverge in several seeds.
    As $\lambda$ increases, \algo enables stronger generalization propagation, resulting in higher and probably divergent learned Q values. Mild generalization propagation plays a crucial role in achieving strong performance. 
    }
    \vspace{-2mm}
    \label{fig:lam}
\end{figure}

\begin{figure}[t]
    \centering
    \includegraphics[width=\textwidth]{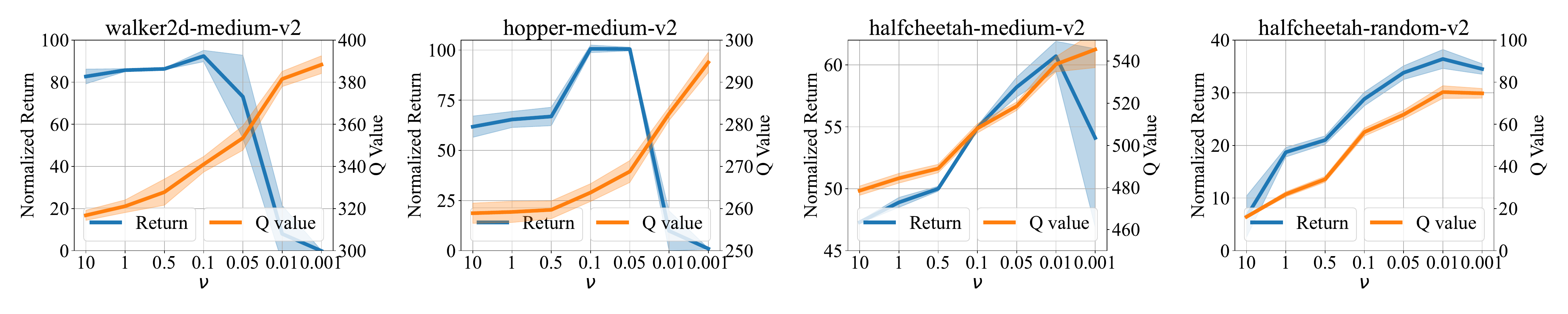}
    \vspace{-6mm}
    \caption{
    Performance and Q values of \algo with varying penalty coefficient $\nu$ over 5 random seeds.
    As $\nu$ decreases, \algo allows broader action generalization, leading to larger learned Q values. Mild action generalization is also critical for attaining superior performance.
    }
    \vspace{-2mm}
    \label{fig:nu}
\end{figure}

\paragraph{Mixture coefficient $\lambda$.} The mixture coefficient $\lambda$ controls the extent of generalization propagation. We fix $\nu=0.1$ and vary $\lambda \in [0,1]$, presenting the learned Q values and performance on several tasks in \cref{fig:lam}. As $\lambda$ increases, \algo enables increased generalization propagation through bootstrapping, and the learned Q values become larger and probably diverge. A moderate $\lambda$ (mild generalization propagation) is crucial for achieving strong performance across datasets.
Under the same degree of action generalization, mild generalization propagation effectively suppresses value overestimation, facilitating more stable policy learning.

\paragraph{Penalty coefficient $\nu$.} The penalty coefficient $\nu$ regulates the degree of action generalization. We fix $\lambda=0.25$ and vary $\nu$. The results are shown in \cref{fig:nu}. As $\nu$ decreases, \algo allows broader action generalization beyond the dataset, which results in higher learned values. Regarding performance, a moderate $\nu$ (mild action generalization) is also crucial for achieving superior performance.

\subsection{Online Fine-tuning after Offline RL}

\begin{wraptable}{r}{0.55\textwidth}
\vspace{-7mm}
\begin{center}
\begin{small}
\setlength{\tabcolsep}{4.5pt}
\caption{Online fine-tuning results on AntMaze tasks, showing normalized scores of offline training and 1M steps online fine-tuning, averaged over 5 seeds.}
\label{tab:online finetune}
\vspace{2mm}
\begin{tabular}{lccc}
\toprule
Dataset-v2 & TD3 & IQL & \algo~(Ours) \\ \midrule
antmaze-u & 0.0 & 89.6 $\rightarrow$ 96.2 & 92.4 $\rightarrow$ \textbf{98.4} \\
antmaze-u-d & 0.0 & 65.6 $\rightarrow$ 62.2 & 75.4 $\rightarrow$ \textbf{89.2} \\
antmaze-m-p & 0.0 & 76.4 $\rightarrow$ 89.8 & 80.2 $\rightarrow$ \textbf{96.8} \\
antmaze-m-d & 0.0 & 72.8 $\rightarrow$ 90.2 & 77.2 $\rightarrow$ \textbf{96.2} \\
antmaze-l-p & 0.0 & 42.0 $\rightarrow$ 78.6 & 55.4 $\rightarrow$ \textbf{86.8} \\
antmaze-l-d & 0.0 & 46.0 $\rightarrow$ 73.4 & 58.8 $\rightarrow$ \textbf{89.0} \\
\midrule
antmaze total & 0.0 & 392.4 $\rightarrow$ 490.4 & 439.4 $\rightarrow$ \textbf{556.4} \\
\bottomrule
\end{tabular}
\end{small}
\end{center}
\vspace{-5pt}
\end{wraptable}

Benefiting from its flexibility in both generalization aspects, \algo enjoys a seamless transition from offline to online learning. This is accomplished through a gradual enhancement of both action generalization and generalization propagation.
Since IQL~\cite{kostrikov2022offline} has demonstrated superior online fine-tuning performance compared to previous methods~\cite{nair2020awac,kumar2020conservative} in its paper, we follow the experimental setup of IQL and compare to IQL. We also train online RL algorithm TD3~\cite{fujimoto2018addressing} from scratch for comparison. We use the challenging AntMaze domains~\cite{fu2020d4rl}, given \algo's already high offline performance in Gym locomotion domains. Results are presented in \cref{tab:online finetune}. While online training from scratch fails in the challenging sparse reward AntMaze tasks, \algo initialized with offline pretraining succeeds in learning near-optimal policies, outperforming IQL by a significant margin. Please refer to \cref{app_sec:experimental_details online} for experimental details, and to \cref{app_sec:online curves} for learning curves.

\section{Conclusion and Limitations}
\label{sec:Conclusion and Limitations}

This work scrutinizes offline RL through the lens of generalization and proposes \algo, comprising mild action generalization and mild generalization propagation, to exploit generalization in offline RL appropriately. 
We theoretically analyze \algo in oracle and worst-case generalization scenarios, and empirically demonstrate its SOTA performance in offline training and online fine-tuning experiments.

While our work contributes valuable insights, it also has limitations. The \algo principle is shown to be effective across most scenarios. However, when the function approximator employed is highly compatible with a specific task setting, the learned value functions may generalize well in the entire action space. In such case, \algo may underperform full generalization methods due to conservatism.

\section*{Acknowledgment}
We thank the anonymous reviewers for feedback on an early version of this paper.
This work was supported by the National Key R\&D Program of China under Grant 2018AAA0102801, National Natural Science Foundation of China under Grant 61827804.

\bibliographystyle{plainnat}
\bibliography{neurips_2024.bib}

%%%%%%%%%%%%%%%%%%%%%%%%%%%%%%%%%%%%%%%%%%%%%%%%%%%%%%%%%%%%
\clearpage

\appendix

\newpage

\section{Extended Related Work}
\label{app_sec:related work}
\paragraph{Model-free offline RL.}
In offline RL, a fixed dataset is provided and no further interactions are allowed~\cite{lange2012batch,levine2020offline}. As a result, conventional off-policy RL algorithms suffer from the extrapolation error due to OOD actions and exhibit poor performance~\cite{fujimoto2019off}.
To address this challenge, various offline RL algorithms have been developed, primarily categorized into model-free and model-based approaches.
In model-free solutions, value regularization methods introduce conservatism in value estimation through direct penalization~\cite{kumar2020conservative,kostrikov2021offline,ma2021conservative,xie2021bellman,cheng2022adversarially,shao2023counterfactual,mao2024supported}, or via value ensembles~\cite{an2021uncertainty,bai2022pessimistic,yang2022rorl}. Policy constraint approaches enforce proximity between the trained policy and the behavior policy, either explicitly via divergence penalties~\cite{wu2019behavior,kumar2019stabilizing,jaques2019way,fujimoto2021minimalist,wu2022supported}, implicitly by weighted behavior cloning~\cite{chen2020bail,peng2019advantage,nair2020awac,wang2020critic,mao2023supported}, or directly through specific parameterization of the policy~\cite{fujimoto2019off,ghasemipour2021emaq,zhou2021plas}.
Some recent efforts focus on learning the optimal policy within the dataset's support (known as in-support or in-sample optimal policy) in a theoretically sound manner~\citep{mao2023supported,mao2024supported,wu2022supported}. These approaches are less influenced by the the dataset's average quality.
Another popular branch of algorithms opts for in-sample learning, which formulates the Bellman target without querying the values of any unseen actions~\cite{brandfonbrener2021offline,ma2021offline,kostrikov2022offline,xiao2023the,zhang2023insample,xu2023offline,garg2023extreme}.
Among these, OneStep RL~\cite{brandfonbrener2021offline} evaluates the behavior policy via SARSA~\cite{sutton2018reinforcement} and performs only one step of constrained policy improvement without off-policy evaluation. 
IQL~\cite{kostrikov2022offline} modifies the SARSA update, using expectile regression to approximate an upper expectile of the value distribution and enables multi-step dynamic programming. Following IQL, several recent works such as InAC~\cite{xiao2023the}, IAC~\citep{zhang2023insample}, $\mathcal{X}$QL~\cite{garg2023extreme}, and SQL~\cite{xu2023offline} have developed different in-sample learning frameworks, further enhancing the performance of in-sample learning approaches.
However, this work shows that the performance of in-sample approaches is confined by eschewing generalization beyond the offline dataset. In contrast, the proposed approach \algo utilizes doubly mild generalization to appropriately exploit generalization and achieves significantly stronger performance across tasks.

\paragraph{Model-based offline RL.}
Model-based offline RL methods involve training an environmental dynamics model, from which synthetic data is generated to facilitate policy optimization~\citep{sutton1991dyna, janner2019trust, kaiser2019model}. In the context of offline RL, algorithms such as MOPO~\citep{yu2020mopo} and MOReL~\citep{kidambi2020morel} propose to estimate the uncertainty within the trained model and subsequently impose penalties or constraints on state-action pairs characterized by high uncertainty levels, thus achieving conservatism in the learning process. Some model-based approaches incorporate conservatism in a similar way to those model-free ones. For example, COMBO~\citep{yu2021combo} leverages value penalization, while BREMEN~\citep{matsushima2021deploymentefficient} employs behavior regularization. More recently, MOBILE~\citep{sun2023model} introduces uncertainty quantification via the inconsistency of Bellman estimations within a learned dynamics ensemble. SCAS~\citep{mao2024offline} proposes a generic model-based regularizer that unifies OOD state correction and OOD action suppression in offline RL. However, typical model-based methods often involve heavy computational overhead~\citep{janner2019trust}, and their effectiveness hinges on the accuracy of the trained dynamics model~\citep{moerland2023model}.

Recently, \citet{bose2024offline} explores multi-task offline RL from the perspective of representation learning and introduced a notion of neighborhood occupancy density. The neighborhood occupancy density at a given stata-action pair in the dataset for a source task is defined as the fraction of points in the dataset within a certain distance from that stata-action pair in the representation space. \citet{bose2024offline} use this concept to bound the representational transfer error in the downstream target task.
In contrast, DMG is a wildly compatible idea in offline RL and provides insights into many offline RL methods. DMG balances the need for generalization with the risk of over-generalization in offline RL. Generalization to stata-action pairs in the neighborhood of the dataset corresponds to mild action generalization in the DMG framework.

\section{Proofs}
\label{app_sec:proofs}
In this section, we provide the proofs of all the theories in the paper.

\subsection{Proof of \cref{prop:gen}}
This section presents the formal theorem for the \cref{prop:gen} in the main paper, along with its proof.

We first make several common continuity assumptions for \cref{prop:gen}.
\begin{assumption}[Lipschitz $Q$]
\label{app_ass:lQ}
The learned value function $Q_\theta$ is $K_Q$-Lipschitz and is upper bounded by $\Qmax$. $\forall s \sim \Dcal$, $\forall a_1, a_2 \sim \Acal$, $|Q_\theta(s,a_1)-Q_\theta(s,a_2)| \leq K_Q \|a_1-a_2\|$.
\end{assumption}
\begin{assumption}[Lipschitz $Q$ gradient]
\label{app_ass:lQg}
The learned value function $Q_\theta$ is smooth, i.e, has a $K_g$-Lipschitz continuous gradient. $\forall s \sim \Dcal$, $\forall a_1, a_2 \sim \Acal$, $\|\nabla_\theta Q_\theta(s,a_1)- \nabla_\theta Q_\theta(s,a_2)\| \leq K_g \|a_1-a_2\|$.
\end{assumption}

\begin{assumption}[Bounded $Q$ and $Q$ gradient]
\label{app_ass:BQ}
$\forall s,a$, $|Q_\theta(s,a)| \leq \Qmax$ and $\|\nabla_\theta Q_\theta(s,a)\| \leq g_\mathrm{max}$.
\end{assumption}

\begin{assumption}[Lipschitz $P$]
\label{app_ass:lP}
The transition dynamics $P$ is $K_P$-Lipschitz. $\forall s,s' \sim \Scal$, $\forall a_1,a_2 \sim \Acal$, $|P(s'|s,a_1)-P(s'|s,a_2)| \leq K_P \|a_1-a_2\|$.
\end{assumption}

\begin{assumption}[Lipschitz $R$]
\label{app_ass:lR}
The reward function $R$ is $K_R$-Lipschitz. $\forall s \sim \Scal$, $\forall a_1,a_2 \sim \Acal$, $|R(s|a_1)-R(s,a_2)| \leq K_R \|a_1-a_2\|$.
\end{assumption}

A continuous learned Q function is particularly necessary for the analysis of value function generalization. Since we often use neural networks or linear models to parameterize the value function $Q_\theta$, Assumptions \ref{app_ass:lQ} and \ref{app_ass:lQg} can be relatively easily satisfied~\cite{gouk2021regularisation}. Assumptions \ref{app_ass:BQ}, \ref{app_ass:lP}, and \ref{app_ass:lR} are also common in the theoretical studies of RL~\cite{dufour2013finite,xiong2022deterministic,ran2023policy} and optimization~\cite{boyd2004convex}.

Before we start the proof of \cref{prop:gen}, we prove the following lemma.
\begin{lemma}
\label{app_lem:lT}
$\forall s \sim \Dcal$, $\forall a_1, a_2 \sim \Acal$, $|\mathcal{T}_u Q_\theta(s,a_1)-\mathcal{T}_u Q_\theta(s,a_2)| \leq K_\mathcal{T} \|a_1-a_2\|$.
where $K_\mathcal{T}$ is a positive bounded constant.
\end{lemma}
\begin{proof}
$\forall s \sim \Dcal$, $\forall a_1, a_2 \sim \Acal$,
\begin{align*}
&\left| \mathcal{T}_u Q_\theta(s,a_1)-\mathcal{T}_u Q_\theta(s,a_2) \right| \\
=&\left| R(s,a_1)- R(s,a_2) + \gamma \E_{s'\sim P(\cdot|s,a_1)}\left[\max_{a'\sim u(\cdot|s')} Q(s',a')\right] - \gamma \E_{s'\sim P(\cdot|s,a_2)}\left[\max_{a'\sim u(\cdot|s')} Q(s',a')\right]\right| \\
\leq &\left| R(s,a_1)- R(s,a_2)\right| + \gamma \left| \E_{s'\sim P(\cdot|s,a_1)}\left[\max_{a'\sim u(\cdot|s')} Q(s',a')\right] - \E_{s'\sim P(\cdot|s,a_2)}\left[\max_{a'\sim u(\cdot|s')} Q(s',a')\right]\right| \\
= &\left| R(s,a_1)- R(s,a_2)\right| + \gamma \left| \sum_{s'}\left(P(s'|s,a_1)-P(s'|s,a_2)\right) \max_{a'\sim u(\cdot|s')} Q(s',a') \right| \\
\leq &\left| R(s,a_1)- R(s,a_2)\right| + \gamma  \sum_{s'}\left|\left(P(s'|s,a_1)-P(s'|s,a_2)\right) \right| \left|\max_{a'\sim u(\cdot|s')} Q(s',a') \right| \\
\leq & K_R \|a_1-a_2\| + \gamma  \sum_{s'}K_P \|a_1-a_2\| \Qmax \\
= & (K_R+ \gamma K_P |\mathcal{S}| \Qmax ) \|a_1-a_2\|
\end{align*}
where the last inequality holds by Assumptions \ref{app_ass:BQ}, \ref{app_ass:lP}, and \ref{app_ass:lR}.

Therefore, for any $s \sim \Dcal$, $a_1, a_2 \sim \Acal$, it holds that 
\begin{equation}
|\mathcal{T}_u Q_\theta(s,a_1)-\mathcal{T}_u Q_\theta(s,a_2)| \leq K_\mathcal{T} \|a_1-a_2\|,
\end{equation}
where $K_\mathcal{T} := K_R+ \gamma K_P |\mathcal{S}| \Qmax$ is a positive bounded constant.
\end{proof}

We restate the scenario analyzed in \cref{prop:gen}: $Q_\theta$ is updated to $Q_{\theta'}$ by one gradient step on a single state-action pair $(s,a)\in \Dcal$, which affects the Q-value of an arbitrary state-action pair $(s,\tilde a)\notin \Dcal$. The parameter update is
\begin{equation}
\label{app_eq:theta}
\theta' = \theta + \alpha(\mathcal{T}_{u}Q_\theta(s,a)-Q_\theta(s,a)) \nabla_\theta Q_\theta(s,a)
\end{equation}
where $\alpha$ is the learning rate. 

Now we start the proof of \cref{prop:gen} in the main paper.
\begin{theorem}[\cref{prop:gen}]
\label{app_prop:gen}
Under Assumptions~\ref{app_ass:lQ} to \ref{app_ass:lR}, the following equation holds when the learning rate $\alpha$ is sufficiently small and $\tilde a$ is sufficiently close to $a$:
\begin{equation}
\label{app_eq:gen}
Q_{\theta'}(s,\tilde a) = Q_{\theta}(s,\tilde a) + C_1 \left(\mathcal{T}_{u}Q_{\theta}(s,\tilde a) -Q_{\theta}(s,\tilde a) + C_2\|\tilde a-a\|\right) + \mathcal{O}\left(\|\theta'-\theta\|^2\right)
\end{equation}
where $C_1 \in [0,1]$ and $C_2 \in [- K_Q-K_R- \gamma K_P |\mathcal{S}| \Qmax, K_Q+K_R+ \gamma K_P |\mathcal{S}| \Qmax]$.
\end{theorem}

\begin{proof}
We formalize $Q_{\theta'}(s,\tilde a)$ by Taylor expansion at the parameter $\theta$:
\begin{equation}
\label{app_eq:taylor}
Q_{\theta'}(s,\tilde a) = Q_{\theta}(s,\tilde a) + \nabla_\theta Q_\theta(s,\tilde a)^\top (\theta'-\theta) + \mathcal{O}\left(\|\theta'-\theta\|^2\right)
\end{equation}

By plugging Eq.~\eqref{app_eq:theta} into Eq.~\eqref{app_eq:taylor}, we have
\begin{equation}
\label{app_eq:after tayler1}
Q_{\theta'}(s,\tilde a) = Q_{\theta}(s,\tilde a) + \alpha \nabla_\theta Q_\theta(s,\tilde a)^\top \nabla_\theta Q_\theta(s,a)\left(\mathcal{T}_{u}Q_\theta(s,a)-Q_\theta(s,a)\right) + \mathcal{O}\left(\|\theta'-\theta\|^2\right)
\end{equation}

According to \cref{app_ass:lQ} and \cref{app_lem:lT}, it holds that
\begin{equation}
|Q_\theta(s,\tilde a)-Q_\theta(s,a)| \leq K_Q \|\tilde a-a\|
\end{equation}
\begin{equation}
|\mathcal{T}_u Q_\theta(s,\tilde a)-\mathcal{T}_u Q_\theta(s,a)| \leq K_\mathcal{T} \|\tilde a-a\|
\end{equation}
where $K_\mathcal{T} := K_R+ \gamma K_P |\mathcal{S}| \Qmax$ is a positive bounded constant.

Therefore, 
\begin{align*}
& \left|(\mathcal{T}_u Q_\theta(s,\tilde a)-Q_\theta(s,\tilde a)) - (\mathcal{T}_u Q_\theta(s, a)-Q_\theta(s, a))\right| \\
=& \left|(\mathcal{T}_u Q_\theta(s,\tilde a)- \mathcal{T}_u Q_\theta(s, a)) +(Q_\theta(s, a) - Q_\theta(s,\tilde a))\right| \\
\leq & \left|(\mathcal{T}_u Q_\theta(s,\tilde a)- \mathcal{T}_u Q_\theta(s, a))\right| +\left|(Q_\theta(s, a) - Q_\theta(s,\tilde a))\right| \\
\leq & ~K_\mathcal{T} \|\tilde a-a\| + K_Q \|\tilde a-a\|
\end{align*}

As a result,
\begin{equation*}
\mathcal{T}_u Q_\theta(s, a)-Q_\theta(s, a) \leq \mathcal{T}_u Q_\theta(s,\tilde a)-Q_\theta(s,\tilde a) + (K_Q+K_\mathcal{T}) \|\tilde a-a\|
\end{equation*}
\begin{equation*}
\mathcal{T}_u Q_\theta(s, a)-Q_\theta(s, a) \geq \mathcal{T}_u Q_\theta(s,\tilde a)-Q_\theta(s,\tilde a) - (K_Q+K_\mathcal{T}) \|\tilde a-a\|
\end{equation*}

Thus we can let
\begin{equation}
\label{app_eq:C_2}
\mathcal{T}_u Q_\theta(s, a)-Q_\theta(s, a) = \mathcal{T}_u Q_\theta(s,\tilde a)-Q_\theta(s,\tilde a) + C_2 \|\tilde a-a\|, 
\end{equation}
where $C_2 \in [- K_Q-K_\mathcal{T}, K_Q+K_\mathcal{T}]$ is a bounded constant.

Now we shift our focus to $\alpha \nabla_\theta Q_\theta(s,\tilde a)^\top \nabla_\theta Q_\theta(s,a)$.
Let $v=\nabla_\theta Q_\theta(s,\tilde a)- \nabla_\theta Q_\theta(s,a)$. According to the smoothness of $Q_\theta$ in \cref{app_ass:lQg}, it holds that
\begin{equation}
\|v\| = \|\nabla_\theta Q_\theta(s,\tilde a)- \nabla_\theta Q_\theta(s,a)\| \leq K_g \|\tilde a-a\|.
\end{equation}

Therefore, 
\begin{align*}
& \nabla_\theta Q_\theta(s,\tilde a)^\top \nabla_\theta Q_\theta(s,a) \\
= & (\nabla_\theta Q_\theta(s, a)+v)^\top \nabla_\theta Q_\theta(s,a) \\
= & \|\nabla_\theta Q_\theta(s,a)\|^2 + v^\top \nabla_\theta Q_\theta(s,a) \\
\geq & \|\nabla_\theta Q_\theta(s,a)\|^2 - \|v\| \|\nabla_\theta Q_\theta(s,a)\| \\
\geq & \|\nabla_\theta Q_\theta(s,a)\|^2 - K_g \|\tilde a-a\| \|\nabla_\theta Q_\theta(s,a)\|
\end{align*}

Therefore, for sufficiently close $\tilde a$ and $a$ such that $\|\tilde a-a\| \leq \|\nabla_\theta Q_\theta(s,a)\|/K_g$, it holds that $\alpha \nabla_\theta Q_\theta(s,\tilde a)^\top \nabla_\theta Q_\theta(s,a) \geq 0$. 

On the other hand, because $\|\nabla_\theta Q_\theta\|$ is bounded by $g_\mathrm{max}$ according to \cref{app_ass:BQ}, it holds that
\begin{align*}
&\alpha \nabla_\theta Q_\theta(s,\tilde a)^\top \nabla_\theta Q_\theta(s,a) \leq \alpha  g_\mathrm{max}^2
\end{align*}

By choosing a small learning rate $\alpha$ such that $\alpha \leq 1/g_\mathrm{max}^2$,
\begin{align*}
&\alpha \nabla_\theta Q_\theta(s,\tilde a)^\top \nabla_\theta Q_\theta(s,a) \leq 1
\end{align*}

In such cases (sufficiently close $\tilde a$ and $a$, and sufficiently small $\alpha$), let 
\begin{equation}
\label{app_eq:C_1}
C_1:=\alpha \nabla_\theta Q_\theta(s,\tilde a)^\top \nabla_\theta Q_\theta(s,a)
\end{equation}
We have $C_1 \in [0,1]$.

By plugging \cref{app_eq:C_2,app_eq:C_1} into \cref{app_eq:after tayler1}, the following equation holds.
\begin{equation}
Q_{\theta'}(s,\tilde a) = Q_{\theta}(s,\tilde a) + C_1 \left(\mathcal{T}_{u}Q_{\theta}(s,\tilde a) -Q_{\theta}(s,\tilde a) + C_2\|\tilde a-a\|\right) + \mathcal{O}\left(\|\theta'-\theta\|^2\right)
\end{equation}
where $C_1 \in [0,1]$, $C_2 \in [- K_Q-K_\mathcal{T}, K_Q+K_\mathcal{T}]$, and $K_\mathcal{T} = K_R+ \gamma K_P |\mathcal{S}| \Qmax$. 

This concludes the proof.

\end{proof}

\subsection{Proofs under Oracle Generalization}
We first restate the several definitions in the main paper.
\begin{definition}[Mildly generalized policy, \cref{def:mg}]
\label{app_def:mg}
Policy $\tilde \beta$ is termed a mildly generalized policy if it satisfies
\begin{equation}
\mathrm{supp}(\hat \beta(\cdot|s)) \subseteq \mathrm{supp}(\tilde \beta(\cdot|s)), ~~\text{and}~~ \max_{a_1 \sim \tilde \beta(\cdot|s)} \min_{a_2 \sim \hat \beta(\cdot|s)} \|a_1-a_2\| \leq \epsilon_a,
\end{equation}
where $\hat \beta$ is the empirical behavior policy in the offline dataset.
\end{definition}

\begin{definition}[\cref{def:algo operator}]
\label{app_def:algo operator}
The Doubly Mildly Generalization (\algo) operator is defined as
\begin{equation}
\mathcal{T}_{\mathrm{\algo}} Q(s,a):=R(s,a)+\gamma \E_{s'\sim P(\cdot|s,a)}\left[\lambda \max_{a'\sim \tilde \beta(\cdot|s')} Q(s',a') + (1-\lambda)\max_{a'\sim \hat \beta(\cdot|s')} Q(s',a')\right]
\end{equation}
where $\hat \beta$ is the empirical behavior policy in the dataset and $\tilde \beta$ is a mildly generalized policy.
\end{definition}

\begin{definition}[\cref{def:in operator}]
The In-sample Q Learning operator~\cite{kostrikov2022offline} is defined as
\begin{equation}
\mathcal{T}_{\mathrm{In}} Q(s,a):=R(s,a)+\gamma \E_{s'\sim P(\cdot|s,a)}\left[\max_{a'\sim \hat \beta(\cdot|s')} Q(s',a')\right]
\end{equation}
where $\hat \beta$ is the empirical behavior policy in the dataset.
\end{definition}

In this subsection, we assume that the learned value function can make oracle generalization in the mild generalization area $\tilde \beta(a|s)>0$, which is formally defined as follows. 
\begin{assumption}[Oracle generalization, \cref{ass:Oracle generalization}]
\label{app_ass:Oracle generalization}
The generalization of learned Q functions in the mild generalization area $\tilde \beta(a|s)>0$ reflects the true value updates according to $\mathcal{T}_{\mathrm{\algo}}$. In other words, $\mathcal{T}_{\mathrm{\algo}}$ is well defined in the mild generalization area $\tilde \beta(a|s)>0$.
\end{assumption}
This assumption can be considered reasonable according to the results presented in \cref{app_prop:gen} above.
In such cases, we can analyze the dynamic programming properties of operators $\mathcal{T}_{\mathrm{In}}$ and $\mathcal{T}_{\mathrm{\algo}}$.

Before we start the proofs of \cref{lem:in Contraction,prop:Contraction} in the main paper, we prove a lemma.
\begin{lemma}
\label{app_lem:max ineq}
For any function $f_1$, $f_2$, any variant $x \in \mathcal{X}$, the following inequality holds:
\begin{equation}
\label{app_eq:max ineq}
\left|\max_{x\in \mathcal{X}}f_1(x)-\max_{x\in \mathcal{X}}f_2(x)\right| \leq \max_{x\in \mathcal{X}} \left|f_1(x)-f_2(x)\right|.
\end{equation}
\end{lemma}

\begin{proof}
Define $x_1:=\argmax_{x\in \mathcal{X}}f_1(x)$ and $x_2:=\argmax_{x\in \mathcal{X}}f_2(x)$.

According to the definition, the following inequality holds:
\begin{equation}
f_1(x_2) -f_2(x_2) \leq f_1(x_1) -f_2(x_2) \leq f_1(x_1) -f_2(x_1)
\end{equation}

Therefore,
\begin{align*}
&\left|\max_{x\in \mathcal{X}}f_1(x)-\max_{x\in \mathcal{X}}f_2(x)\right| \\
=&\left|f_1(x_1) -f_2(x_2)\right| \\
\leq &\max \left\{\left|f_1(x_2) -f_2(x_2)\right|, \left|f_1(x_1) -f_2(x_1)\right| \right\}\\
\leq &\max_{x\in \mathcal{X}} \left|f_1(x) -f_2(x)\right|
\end{align*}

This concludes the proof of \cref{app_lem:max ineq}.
\end{proof}

\begin{lemma}[\cref{lem:in Contraction}]
\label{app_lem:in Contraction}
$\mathcal{T}_{\mathrm{In}}$ is a $\gamma$-contraction operator in the in-sample area $\hat \beta(a|s)>0$ under the $\mathcal{L}_\infty$ norm.
\end{lemma}

\begin{proof}
Let $f_1$ and $f_2$ be two arbitrary functions.

For all $(s,a) ~s.t. ~\hat \beta(a|s)>0$, we have
\begin{align*}
& \left| \mathcal{T}_{\mathrm{In}}f_1(s,a)-\mathcal{T}_{\mathrm{In}}f_2(s,a) \right| \\
= & \left|R(s,a)+\gamma \E_{s'\sim P(\cdot|s,a)}\left[\max_{a'\sim \hat \beta(\cdot|s')} f_1(s',a')\right]  - R(s,a)-\gamma \E_{s'\sim P(\cdot|s,a)}\left[\max_{a'\sim \hat \beta(\cdot|s')} f_2(s',a')\right]\right| \\
=&  \gamma \left|\mathbb{E}_{s'\sim P(\cdot|s,a)}\left[\max_{a'\sim \hat \beta(\cdot|s')}f_1(s',a')- \max_{a'\sim \hat \beta(\cdot|s')}f_2(s',a')\right]\right| \\
\leq&  \gamma \mathbb{E}_{s'\sim P(\cdot|s,a)}\left[\left|\max_{a'\sim \hat \beta(\cdot|s')}f_1(s',a')- \max_{a'\sim \hat \beta(\cdot|s')}f_2(s',a')\right|\right] \\
\leq&  \gamma \mathbb{E}_{s'\sim P(\cdot|s,a)}\left[\max_{a'\sim \hat \beta(\cdot|s')} \left|f_1(s',a')-f_2(s',a')\right|\right] \\
\leq &\gamma \max_{(s,a): \hat\beta(a|s)>0} |f_1(s,a)-f_2(s,a)|
\end{align*}
where the second inequality holds by \cref{app_lem:max ineq}.

Therefore, in the in-sample area $\tilde \beta(a|s)>0$, $\mathcal{T}_{\mathrm{In}}$ is a $\gamma$-contraction operator under the $\mathcal{L}_\infty$ norm. This concludes the proof for $\mathcal{T}_{\mathrm{In}}$.
\end{proof}

Thus, by repeatedly applying $\mathcal{T}_{\mathrm{In}}$, any initial Q function can converge to the unique fixed point $Q^*_{\mathrm{In}}$. We denote its induced policy by $\pi^*_{\mathrm{In}}$:
\begin{equation}
Q^*_{\mathrm{In}}(s,a)=R(s,a)+\gamma \E_{s'\sim P(\cdot|s,a)}\left[\max_{a'\sim \hat \beta(\cdot|s')} Q^*_{\mathrm{In}}(s',a')\right], \quad \hat\beta(a|s)>0,
\end{equation}
\begin{equation}
\pi^*_{\mathrm{In}}(s) := \argmax_{a\sim \hat\beta(\cdot|s)} Q^*_{\mathrm{In}}(s,a).
\end{equation}

Here, $Q^*_{\mathrm{In}}$ is known as the in-sample optimal value function~\cite{kumar2019stabilizing,kostrikov2022offline}, which is the value function of the in-sample optimal policy $\pi^*_{\mathrm{In}}$. We refer readers to \cite{wu2022supported,kostrikov2022offline,mao2023supported,mao2024supported} for more discussions on the in-sample or in-support optimality.

Now we start the proof of \cref{prop:Contraction} in the main paper.
\begin{theorem}[Contraction, \cref{prop:Contraction}]
\label{app_thm:Contraction}
Under \cref{app_ass:Oracle generalization}, $\mathcal{T}_{\mathrm{\algo}}$ is a $\gamma$-contraction operator in the mild generalization area $\tilde \beta(a|s)>0$ under the $\mathcal{L}_\infty$ norm.
Therefore, by repeatedly applying $\mathcal{T}_{\mathrm{\algo}}$, any initial Q function can converge to the unique fixed point $Q^*_{\mathrm{\algo}}$.
\end{theorem}

\begin{proof}
By the oracle generalization assumption (\cref{app_ass:Oracle generalization}), $\mathcal{T}_{\mathrm{\algo}}$ is well defined in the mild generalization area $\tilde \beta(a|s)>0$.

Let $f_1$ and $f_2$ be two arbitrary functions. For all $(s,a) ~s.t. ~\tilde \beta(a|s)>0$, we have
\begin{align*}
& \mathcal{T}_{\mathrm{\algo}}f_1(s,a)-\mathcal{T}_{\mathrm{\algo}}f_2(s,a) \\
= & R(s,a)+\gamma \E_{s'\sim P(\cdot|s,a)}\left[\lambda \max_{a'\sim \tilde \beta(\cdot|s')} f_1(s',a') + (1-\lambda)\max_{a'\sim \hat \beta(\cdot|s')} f_1(s',a')\right] \\
& - R(s,a)-\gamma \E_{s'\sim P(\cdot|s,a)}\left[\lambda \max_{a'\sim \tilde \beta(\cdot|s')} f_2(s',a') + (1-\lambda)\max_{a'\sim \hat \beta(\cdot|s')} f_2(s',a')\right] \\
=& \gamma \lambda \mathbb{E}_{s'\sim P(\cdot|s,a)}\left[\max_{a'\sim \tilde \beta(\cdot|s')}f_1(s',a')- \max_{a'\sim \tilde \beta(\cdot|s')}f_2(s',a')\right] \\
& + \gamma(1-\lambda) \mathbb{E}_{s'\sim P(\cdot|s,a)}\left[\max_{a'\sim \hat \beta(\cdot|s')}f_1(s',a')- \max_{a'\sim \hat \beta(\cdot|s')}f_2(s',a')\right]
\end{align*}

Therefore, for all $(s,a) ~s.t. ~\tilde \beta(a|s)>0$,
\begin{align*}
&|\mathcal{T}_{\mathrm{\algo}}f_1(s,a)-\mathcal{T}_{\mathrm{\algo}}f_2(s,a)| \\
\leq & \left|\gamma \lambda \mathbb{E}_{s'\sim P(\cdot|s,a)}\left[\max_{a'\sim \tilde \beta(\cdot|s')}f_1(s',a')- \max_{a'\sim \tilde \beta(\cdot|s')}f_2(s',a')\right]\right| \\
& + \left|\gamma(1-\lambda) \mathbb{E}_{s'\sim P(\cdot|s,a)}\left[\max_{a'\sim \hat \beta(\cdot|s')}f_1(s',a')- \max_{a'\sim \hat \beta(\cdot|s')}f_2(s',a')\right]\right| \\
\leq & \gamma \lambda \mathbb{E}_{s'\sim P(\cdot|s,a)}\left[\left|\max_{a'\sim \tilde \beta(\cdot|s')}f_1(s',a')- \max_{a'\sim \tilde \beta(\cdot|s')}f_2(s',a')\right|\right] \\
& + \gamma(1-\lambda) \mathbb{E}_{s'\sim P(\cdot|s,a)}\left[\left|\max_{a'\sim \hat \beta(\cdot|s')}f_1(s',a')- \max_{a'\sim \hat \beta(\cdot|s')}f_2(s',a')\right|\right] \\
\leq & \gamma \lambda \mathbb{E}_{s'\sim P(\cdot|s,a)}\left[\max_{a'\sim \tilde \beta(\cdot|s')} \left|f_1(s',a')-f_2(s',a')\right|\right] \\
& + \gamma(1-\lambda) \mathbb{E}_{s'\sim P(\cdot|s,a)}\left[\max_{a'\sim \hat \beta(\cdot|s')} \left|f_1(s',a')-f_2(s',a')\right|\right] \\
\leq & \gamma \lambda \mathbb{E}_{s'\sim P(\cdot|s,a)} \max_{(s,a): \tilde\beta(a|s)>0} |f_1(s,a)-f_2(s,a)| \\
& + \gamma(1-\lambda) \mathbb{E}_{s'\sim P(\cdot|s,a)} \max_{(s,a): \tilde\beta(a|s)>0} |f_1(s,a)-f_2(s,a)| \\
= &\gamma \max_{(s,a): \tilde\beta(a|s)>0} |f_1(s,a)-f_2(s,a)|
\end{align*}
where the third inequality holds by \cref{app_lem:max ineq}. 

Therefore, in the mild generalization area $\tilde \beta(a|s)>0$, $\mathcal{T}_{\mathrm{\algo}}$ is a $\gamma$-contraction operator under the $\mathcal{L}_\infty$ norm. 
This concludes the proof.
\end{proof}

As a result, by repeatedly applying $\mathcal{T}_{\mathrm{\algo}}$, any initial Q function can converge to the unique fixed point $Q^*_{\mathrm{\algo}}$. We denote the induced policy of $Q^*_{\mathrm{\algo}}$ by $\pi^*_{\mathrm{\algo}}$.
\begin{equation}
Q^*_{\mathrm{\algo}}(s,a)=R(s,a)+\gamma \E_{s'\sim P(\cdot|s,a)}\left[\max_{a'\sim \tilde \beta(\cdot|s')} Q^*_{\mathrm{\algo}}(s',a')\right], \quad \tilde\beta(a|s)>0,
\end{equation}
\begin{equation}
\pi^*_{\mathrm{\algo}}(s) := \argmax_{a\sim \tilde\beta(\cdot|s)} Q^*_{\mathrm{\algo}}(s,a).
\end{equation}

Before we start the proof of \cref{thm:Performance}, we prove two lemmas.
\begin{lemma}
\label{app_lem:T ineq1}
Under \cref{app_ass:Oracle generalization}, for any function $f$, the following inequality holds:
\begin{equation}
\mathcal{T}_{\mathrm{\algo}} f(s,a) \geq \mathcal{T}_{\mathrm{In}} f(s,a), ~~\forall (s,a) ~s.t. ~\tilde \beta(a|s)>0.
\end{equation}
\end{lemma}
\begin{proof}
The oracle generalization assumption (\cref{app_ass:Oracle generalization}) implies that $\mathcal{T}_{\mathrm{In}}$ is also well defined in the mild generalization area $\tilde \beta(a|s)>0$.
Because $\mathrm{supp}(\hat \beta(\cdot|s)) \subseteq \mathrm{supp}(\tilde \beta(\cdot|s))$, $\mathcal{T}_{\mathrm{In}}$ requires less information than $\mathcal{T}_{\mathrm{\algo}}$. Therefore, $\mathcal{T}_{\mathrm{\algo}}$ being well defined in the mild generalization area implies $\mathcal{T}_{\mathrm{In}}$ also being well defined in that area.

According to the definitions, for all $(s,a) ~s.t. ~\tilde \beta(a|s)>0$,
\begin{equation}
\mathcal{T}_{\mathrm{\algo}} f(s,a)=R(s,a)+\gamma \E_{s'\sim P(\cdot|s,a)}\left[\lambda \max_{a'\sim \tilde \beta(\cdot|s')} f(s',a') + (1-\lambda)\max_{a'\sim \hat \beta(\cdot|s')} f(s',a')\right]
\end{equation}
\begin{equation}
\mathcal{T}_{\mathrm{In}} f(s,a)=R(s,a)+\gamma \E_{s'\sim P(\cdot|s,a)}\left[\max_{a'\sim \hat \beta(\cdot|s')} f(s',a')\right]
\end{equation}

Therefore, for all $(s,a) ~s.t. ~\tilde \beta(a|s)>0$, we have
\begin{align*}
& \mathcal{T}_{\mathrm{\algo}}f(s,a)-\mathcal{T}_{\mathrm{In}}f(s,a) \\
= & \gamma \E_{s'\sim P(\cdot|s,a)}\left[\lambda \max_{a'\sim \tilde \beta(\cdot|s')} f(s',a') -\lambda \max_{a'\sim \hat \beta(\cdot|s')} f(s',a')\right] \\
\geq & 0
\end{align*}
where the last inequality holds because $\tilde \beta$ has a wider support than $\hat \beta$.
\end{proof}

\begin{lemma}
\label{app_lem:T ineq2}
Under \cref{app_ass:Oracle generalization}, for any function $f_1,f_2$ such that $f_1(s,a) \geq f_2(s,a)$, $\forall (s,a) ~s.t. ~\tilde \beta(a|s)>0$, the following inequality holds:
\begin{equation}
\mathcal{T}_{\mathrm{\algo}} f_1(s,a) \geq \mathcal{T}_{\mathrm{\algo}} f_2(s,a), ~~\forall (s,a) ~s.t. ~\tilde \beta(a|s)>0
\end{equation}
\end{lemma}
\begin{proof}
By \cref{app_ass:Oracle generalization}, $\mathcal{T}_{\mathrm{\algo}}$ is well defined in the mild generalization area $\tilde \beta(a|s)>0$.

According to the definition, for all $(s,a) ~s.t. ~\tilde \beta(a|s)>0$,
\begin{equation}
\mathcal{T}_{\mathrm{\algo}} f(s,a)=R(s,a)+\gamma \E_{s'\sim P(\cdot|s,a)}\left[\lambda \max_{a'\sim \tilde \beta(\cdot|s')} f(s',a') + (1-\lambda)\max_{a'\sim \hat \beta(\cdot|s')} f(s',a')\right]
\end{equation}

$f_1$ and $f_2$ satisfy
\begin{equation}
f_1(s,a) \geq f_2(s,a), \forall (s,a) ~~s.t. ~\tilde \beta(a|s)>0.
\end{equation}

Therefore, for all $(s,a) ~s.t. ~\tilde \beta(a|s)>0$,
\begin{align*}
& \mathcal{T}_{\mathrm{\algo}}f_1(s,a)-\mathcal{T}_{\mathrm{\algo}}f_2(s,a) \\
= & \gamma \E_{s'\sim P(\cdot|s,a)}\left[\lambda \max_{a'\sim \tilde \beta(\cdot|s')} f_1(s',a') -\lambda \max_{a'\sim \tilde \beta(\cdot|s')} f_2(s',a')\right] \\
& + \gamma \E_{s'\sim P(\cdot|s,a)}\left[(1-\lambda) \max_{a'\sim \hat \beta(\cdot|s')} f_1(s',a') -(1-\lambda) \max_{a'\sim \hat \beta(\cdot|s')} f_2(s',a')\right]\\
\geq & 0
\end{align*}
\end{proof}

Now we start the proof of \cref{thm:Performance} in the main paper.
\begin{theorem}[Performance, \cref{thm:Performance}]
\label{app_thm:Performance}
Under \cref{app_ass:Oracle generalization}, the value functions of $\pi^*_{\mathrm{\algo}}$ and $\pi^*_{\mathrm{In}}$ satisfy:
\begin{equation}
V^{\pi^*_{\mathrm{\algo}}}(s) \geq V^{\pi^*_{\mathrm{In}}}(s), \quad \forall s \in \Dcal.
\end{equation}
\end{theorem}

\begin{proof}

We first prove the following inequality:
\begin{equation}
\label{app_eq:perf proof}
(\mathcal{T}_{\mathrm{\algo}})^k f(s,a) \geq (\mathcal{T}_{\mathrm{In}})^k f(s,a), ~\forall k \in \mathbb{Z}^+, ~\forall f,~~\forall (s,a) ~s.t. ~\tilde\beta(a|s)>0.
\end{equation}

When $k=1$, according to \cref{app_lem:T ineq1}, it holds that
\begin{equation*}
(\mathcal{T}_{\mathrm{\algo}})^1 f(s,a) \geq (\mathcal{T}_{\mathrm{In}})^1 f(s,a),  ~\forall f,~~\forall (s,a) ~s.t. ~\tilde \beta(a|s)>0.
\end{equation*}

Suppose when $k=i$, the following inequality holds:
\begin{equation*}
(\mathcal{T}_{\mathrm{\algo}})^i f(s,a) \geq (\mathcal{T}_{\mathrm{In}})^i f(s,a),  ~\forall f,~~\forall (s,a) ~s.t. ~\tilde \beta(a|s)>0.
\end{equation*}

Then $(\mathcal{T}_{\mathrm{\algo}})^i f$ and $(\mathcal{T}_{\mathrm{In}})^i f$ are the two functions $f_1,f_2$ that satisfy the condition in \cref{app_lem:T ineq2}. Therefore, by \cref{app_lem:T ineq2}, it holds that
\begin{equation}
\label{app_eq:fixed point proof}
\mathcal{T}_{\mathrm{\algo}}(\mathcal{T}_{\mathrm{\algo}})^i f(s,a) \geq \mathcal{T}_{\mathrm{\algo}} (\mathcal{T}_{\mathrm{In}})^i f(s,a),  ~\forall f,~~\forall (s,a) ~s.t. ~\tilde \beta(a|s)>0.
\end{equation}

Now considering $(\mathcal{T}_{\mathrm{In}})^i f$ as the function $f$ in \cref{app_lem:T ineq1}. By \cref{app_lem:T ineq1}, it holds that
\begin{equation}
\label{app_eq:fixed point proof1}
\mathcal{T}_{\mathrm{\algo}} (\mathcal{T}_{\mathrm{In}})^i f(s,a) \geq \mathcal{T}_{\mathrm{In}} (\mathcal{T}_{\mathrm{In}})^i f(s,a),  ~\forall f,~~\forall (s,a) ~s.t. ~\tilde \beta(a|s)>0.
\end{equation}

Combining \cref{app_eq:fixed point proof,app_eq:fixed point proof1}, we have
\begin{equation*}
(\mathcal{T}_{\mathrm{\algo}})^{i+1} f(s,a) \geq (\mathcal{T}_{\mathrm{In}})^{i+1} f(s,a),  ~\forall f,~~\forall (s,a) ~s.t. ~\tilde \beta(a|s)>0.
\end{equation*}

Therefore, for all $k \in \mathbb{Z}^+$, the following inequality holds:
\begin{equation}
\label{app_eq:perf proof1}
(\mathcal{T}_{\mathrm{\algo}})^k f(s,a) \geq (\mathcal{T}_{\mathrm{In}})^k f(s,a),  ~\forall f,~~\forall (s,a) ~s.t. ~\tilde \beta(a|s)>0.
\end{equation}

\cref{app_lem:in Contraction} states that $\mathcal{T}_{\mathrm{In}}$ is a $\gamma$-contraction operator in the in-sample area $\hat \beta(a|s)>0$. Thus we have
\begin{equation}
\label{app_eq:perf proof2}
Q^*_{\mathrm{In}}(s,a) = \lim_{k \to \infty}(\mathcal{T}_{\mathrm{In}})^k f(s,a), ~\forall (s,a) ~s.t. ~\hat \beta(a|s)>0.
\end{equation}

Under \cref{app_ass:Oracle generalization}, \cref{app_thm:Contraction} states that $\mathcal{T}_{\mathrm{\algo}}$ is a $\gamma$-contraction operator in the mild generalization area $\tilde \beta(a|s)>0$. Thus we have
\begin{equation}
\label{app_eq:perf proof3}
Q^*_{\mathrm{\algo}}(s,a) = \lim_{k \to \infty} (\mathcal{T}_{\mathrm{\algo}})^k f(s,a), ~\forall (s,a) ~s.t. ~\tilde \beta(a|s)>0.
\end{equation}

As $\tilde \beta$ has a wider support than $\hat \beta$, $\mathrm{supp}(\hat \beta(\cdot|s)) \subseteq \mathrm{supp}(\tilde \beta(\cdot|s))$, the following inequality holds by combining \cref{app_eq:perf proof1,app_eq:perf proof2,app_eq:perf proof3}:
\begin{equation}
\label{app_eq:perf proof4}
Q^*_{\mathrm{\algo}}(s,a) \geq Q^*_{\mathrm{In}}(s,a), ~\forall (s,a) ~s.t. ~\hat \beta(a|s)>0.
\end{equation}

Therefore, for any $s \sim\Dcal$,
\begin{align*}
& V^{\pi^*_{\mathrm{\algo}}}(s) = V^*_{\mathrm{\algo}}(s) = Q^*_{\mathrm{\algo}}(s,\pi^*_{\mathrm{\algo}}(s)) \\
\geq & Q^*_{\mathrm{\algo}}(s,\pi^*_{\mathrm{In}}(s)) \\
\geq&  Q^*_{\mathrm{In}}(s,\pi^*_{\mathrm{In}}(s))
= V^*_{\mathrm{In}}(s) = V^{\pi^*_{\mathrm{In}}}(s)
\end{align*}
where the first inequality holds because $\pi^*_{\mathrm{\algo}}(s) := \argmax_{a\sim \tilde\beta(\cdot|s)} Q^*_{\mathrm{\algo}}(s,a)$ and $\pi^*_{\mathrm{In}}(s) \in \hat \beta(\cdot|s)$ (thus $\pi^*_{\mathrm{In}}(s) \in \tilde \beta(\cdot|s)$), and the second inequality holds by \cref{app_eq:perf proof4}.

This concludes the proof.
\end{proof}

\cref{app_thm:Performance} indicates that the policy induced by the \algo operator can behave better than the in-sample optimal policy under the oracle generalization condition.

\subsection{Proofs under Worst-case Generalization}
In this section, we focus on the analyses in the worst-case generalization scenario, where the learned value functions may exhibit poor generalization in the mild generalization area $\tilde \beta(a|s)>0$. In other words, this section considers that $\mathcal{T}_{\mathrm{\algo}}$ is only defined in the in-sample area $\hat \beta(a|s)>0$ and the learned value functions may have any generalization error at other state-action pairs. In this case, we use the notation $\hat{\mathcal{T}}_{\mathrm{\algo}}$ for differentiation.

In this case, we make the following continuity assumptions about the learned $Q$ function and the transition dynamics $P$.

\begin{assumption}[Lipschitz $Q$]
\label{app_ass:lQ1}
The learned Q function is $K_Q$-Lipschitz. $\forall s \sim \Dcal$, $\forall a_1, a_2 \sim \Acal$, $|Q(s,a_1)-Q(s,a_2)| \leq K_Q \|a_1-a_2\|$
\end{assumption}

\begin{assumption}[Lipschitz $P$]
\label{app_ass:lP1}
The transition dynamics $P$ is $K_P$-Lipschitz. $\forall s,s' \sim \Scal$, $\forall a_1,a_2 \sim \Acal$, $|P(s'|s,a_1)-P(s'|s,a_2)| \leq K_P \|a_1-a_2\|$
\end{assumption}

For \cref{app_ass:lQ1}, a continuous learned Q function is particularly necessary for the analysis of value function generalization and can be relatively easily satisfied~\cite{gouk2021regularisation}, since we often use neural networks or linear models to parameterize the value function. For \cref{app_ass:lP1}, continuous transition dynamics is also a standard assumption in the theoretical studies of RL~\cite{dufour2013finite,dufour2015approximation,xiong2022deterministic,ran2023policy}. Several previous works assume the transition to be Lipschitz continuous with respect to (w.r.t) both state and action~\cite{dufour2013finite,dufour2015approximation}. In our paper, we need the Lipschitz continuity to hold only w.r.t. action.

Before we start the proof of \cref{prop:Limited over-estimation}, we prove two lemmas.
\begin{lemma}
\label{app_lem:over-e}
Under \cref{app_ass:lQ1}, for any function $f$ and $s \sim \mathcal{D}$, the following inequality holds:
\begin{equation}
\max_{a\sim \tilde\beta(\cdot|s)}f(s,a)-\max_{a\sim \hat\beta(\cdot|s)}f(s,a) \leq \epsilon_a K_Q.
\end{equation}
\end{lemma}
\begin{proof}
For any $s \sim \mathcal{D}$, we define $\tilde a^*$, $\hat a^*$, $\hat a'$ as follows:
\begin{align}
\tilde a^* &= \argmax_{a\sim \tilde\beta(\cdot|s)}f(s,a) \\
\hat a^* &= \argmax_{a\sim \hat\beta(\cdot|s)}f(s,a) \\
\hat a' &= \argmin_{a\sim \hat\beta(\cdot|s)}\|\tilde a^*-a\|
\end{align}

According to the definition of mildly generalized policy $\tilde\beta$ (\cref{app_def:mg}), it holds that $\|\tilde a^*-\hat a'\| \leq \epsilon_a$.
Further by \cref{app_ass:lQ1}, it holds that 
\begin{equation*}
|f(s,\tilde a^*)-f(s,\hat a')| \leq K_Q \|\tilde a^*-\hat a'\| \leq \epsilon_a K_Q,~~ \forall s \sim \mathcal{D}.
\end{equation*}

Therefore,
\begin{equation*}
f(s,\tilde a^*)-f(s,\hat a^*) \leq f(s,\tilde a^*)-f(s,\hat a') \leq \epsilon_a K_Q,~~ \forall s \sim \mathcal{D}.
\end{equation*}
\end{proof}

\begin{lemma}
\label{app_lem:T ineq3}
For any function $f_1,f_2$ such that $f_1(s,a) \geq f_2(s,a)$, $\forall (s,a) ~s.t. ~\hat \beta(a|s)>0$, the following inequality holds:
\begin{equation}
\mathcal{T}_{\mathrm{In}} f_1(s,a) \geq \mathcal{T}_{\mathrm{In}} f_2(s,a), ~~\forall (s,a) ~s.t. ~\hat \beta(a|s)>0.
\end{equation}
\end{lemma}

\begin{proof}
According to the definitions, for all $(s,a) ~s.t. ~\hat \beta(a|s)>0$,
\begin{equation}
\mathcal{T}_{\mathrm{In}} f(s,a)=R(s,a)+\gamma \E_{s'\sim P(\cdot|s,a)}\left[\max_{a'\sim \hat \beta(\cdot|s')} f(s',a')\right]
\end{equation}

$f_1$ and $f_2$ satisfy
\begin{equation*}
f_1(s,a) \geq f_2(s,a), \forall (s,a) ~~s.t. ~\hat \beta(a|s)>0.
\end{equation*}

Therefore, for all $(s,a) ~s.t. ~\hat \beta(a|s)>0$,
\begin{align*}
& \mathcal{T}_{\mathrm{In}}f_1(s,a)-\mathcal{T}_{\mathrm{In}}f_2(s,a) \\
= & \gamma \E_{s'\sim P(\cdot|s,a)}\left[\max_{a'\sim \hat \beta(\cdot|s')} f_1(s',a') - \max_{a'\sim \hat \beta(\cdot|s')} f_2(s',a')\right]\\
\geq & 0
\end{align*}
\end{proof}

Now we start the proof of \cref{prop:Limited over-estimation} in the main paper.

We consider the iteration starting from arbitrary function $Q^0$: $\hat Q^k_{\mathrm{\algo}}=\hat{\mathcal{T}}_{\mathrm{\algo}} \hat Q^{k-1}_{\mathrm{\algo}}$ and $Q^k_{\mathrm{In}}=\mathcal{T}_{\mathrm{In}} Q^{k-1}_{\mathrm{In}}$, $\forall k \in \mathbb{Z}^+$. The possible value of $\hat Q^k_{\mathrm{\algo}}$ is upper bounded by the following results.
\begin{theorem}[Limited over-estimation, \cref{prop:Limited over-estimation}]
\label{app_prop:Limited over-estimation}
Under \cref{app_ass:lQ1}, the learned Q function of \algo by iterating $\hat{\mathcal{T}}_{\mathrm{\algo}}$ satisfies the following inequality
\begin{equation}
\label{app_eq:over-e}
Q^k_{\mathrm{In}}(s,a) \leq \hat Q^k_{\mathrm{\algo}}(s,a) \leq Q^k_{\mathrm{In}}(s,a) + \frac{\lambda\epsilon_a K_Q \gamma}{1-\gamma}(1-\gamma^k), ~\forall s,a \sim \Dcal, ~\forall k \in \mathbb{Z}^+.
\end{equation}
\end{theorem}

\begin{proof}
Under worst-case generalization, $\hat{\mathcal{T}}_{\mathrm{\algo}}$ is only defined in the area $\hat \beta(a|s)>0$, i.e., the dataset, and may have any generalization error at other $(s,a)$.

For any function $f$ and any $s,a \sim \mathcal{D}$,
\begin{align*}
& \hat{\mathcal{T}}_{\mathrm{\algo}} f(s,a) - \mathcal{T}_{\mathrm{In}} f(s,a) \\
= & R(s,a)+\gamma \E_{s'\sim P(\cdot|s,a)}\left[\lambda \max_{a'\sim \tilde \beta(\cdot|s')} f(s',a') + (1-\lambda)\max_{a'\sim \hat \beta(\cdot|s')} f(s',a')\right] \\
& - R(s,a)-\gamma \E_{s'\sim P(\cdot|s,a)}\left[\max_{a'\sim \hat \beta(\cdot|s')} f(s',a')\right] \\
=& \gamma \E_{s'\sim P(\cdot|s,a)}\left[\lambda \max_{a'\sim \tilde \beta(\cdot|s')} f(s',a') -\lambda\max_{a'\sim \hat \beta(\cdot|s')} f(s',a')\right] \\
\leq& \gamma\lambda \epsilon_a K_Q
\end{align*}
where the last inequality holds by \cref{app_lem:over-e}.

On the other hand, because $\tilde \beta$ has a wider support than $\hat \beta$, we also have
\begin{align*}
\hat{\mathcal{T}}_{\mathrm{\algo}} f(s,a) - \mathcal{T}_{\mathrm{In}} f(s,a) \geq 0
\end{align*}

Therefore, for any function $f$, the following inequality holds:
\begin{equation}
\label{app_eq:over-e proof}
\mathcal{T}_{\mathrm{In}} f(s,a) \leq \hat{\mathcal{T}}_{\mathrm{\algo}} f(s,a) \leq \mathcal{T}_{\mathrm{In}} f(s,a) + \gamma\lambda \epsilon_a K_Q, ~~\forall s,a \sim \Dcal.
\end{equation}

Let $f$ in \cref{app_eq:over-e proof} be $Q^0$. We have
\begin{equation}
Q^1_{\mathrm{In}}(s,a) \leq \hat Q^1_{\mathrm{\algo}}(s,a) \leq Q^1_{\mathrm{In}}(s,a) + \frac{\lambda\epsilon_a K_Q \gamma}{1-\gamma}(1-\gamma), ~\forall s,a \sim \Dcal.
\end{equation}

This is the same as \cref{app_eq:over-e} with $k=1$. Therefore, \cref{app_eq:over-e} holds when $k=1$.

Suppose when $k=i$, \cref{app_eq:over-e} holds:
\begin{equation}
\label{app_eq:over-e proof1}
Q^i_{\mathrm{In}}(s,a) \leq \hat Q^i_{\mathrm{\algo}}(s,a) \leq Q^i_{\mathrm{In}}(s,a) + \frac{\lambda\epsilon_a K_Q \gamma}{1-\gamma}(1-\gamma^i), ~\forall s,a \sim \Dcal.
\end{equation}

Then let $f$ in \cref{app_eq:over-e proof} be $\hat Q^i_{\mathrm{\algo}}$. We have
\begin{equation}
\label{app_eq:over-e proof2}
\mathcal{T}_{\mathrm{In}} \hat Q^i_{\mathrm{\algo}}(s,a) \leq \hat Q^{i+1}_{\mathrm{\algo}}(s,a) = \hat{\mathcal{T}}_{\mathrm{\algo}} \hat Q^{i}_{\mathrm{\algo}}(s,a) \leq \mathcal{T}_{\mathrm{In}} \hat Q^i_{\mathrm{\algo}}(s,a) + \gamma\lambda \epsilon_a K_Q, ~~\forall s,a \sim \Dcal.
\end{equation}

On the one hand, according to \cref{app_lem:T ineq3} and \cref{app_eq:over-e proof1}, for any $s,a \sim \Dcal$, we have
\begin{align}
&\mathcal{T}_{\mathrm{In}} \hat Q^i_{\mathrm{\algo}}(s,a) \nonumber\\
\leq & \mathcal{T}_{\mathrm{In}} \left(Q^i_{\mathrm{In}}(s,a) + \frac{\lambda\epsilon_a K_Q \gamma}{1-\gamma}(1-\gamma^i)\right) \nonumber\\
= & R(s,a)+\gamma \E_{s'\sim P(\cdot|s,a)}\left[\max_{a'\sim \hat \beta(\cdot|s')} \left(Q^i_{\mathrm{In}}(s',a') + \frac{\lambda\epsilon_a K_Q \gamma}{1-\gamma}(1-\gamma^i)\right)\right] \nonumber\\
= & R(s,a)+\gamma \E_{s'\sim P(\cdot|s,a)}\left[\max_{a'\sim \hat \beta(\cdot|s')} Q^i_{\mathrm{In}}(s',a')\right] + \gamma \frac{\lambda\epsilon_a K_Q \gamma}{1-\gamma}(1-\gamma^i) \nonumber\\
= & \mathcal{T}_{\mathrm{In}} Q^i_{\mathrm{In}}(s,a) + \gamma \frac{\lambda\epsilon_a K_Q \gamma}{1-\gamma}(1-\gamma^i) \nonumber\\
=& Q^{i+1}_{\mathrm{In}}(s,a) + \gamma \frac{\lambda\epsilon_a K_Q \gamma}{1-\gamma}(1-\gamma^i) \label{app_eq:over-e proof3}
\end{align}

Combining \cref{app_eq:over-e proof2,app_eq:over-e proof3}, for any $s,a \sim \Dcal$, we have
\begin{align*}
&\hat Q^{i+1}_{\mathrm{\algo}}(s,a) \\
\leq& Q^{i+1}_{\mathrm{In}}(s,a) + \gamma \frac{\lambda\epsilon_a K_Q \gamma}{1-\gamma}(1-\gamma^i) + \gamma\lambda \epsilon_a K_Q \\
=& Q^{i+1}_{\mathrm{In}}(s,a) + \lambda\epsilon_a K_Q \gamma \left(\frac{\gamma(1-\gamma^i)}{1-\gamma} + 1\right) \\
=& Q^{i+1}_{\mathrm{In}}(s,a) + \frac{\lambda\epsilon_a K_Q \gamma}{1-\gamma}(1-\gamma^{i+1})
\end{align*}

On the other hand, according to \cref{app_lem:T ineq3} and \cref{app_eq:over-e proof1}, for any $s,a \sim \Dcal$, we have
\begin{align}
\mathcal{T}_{\mathrm{In}} \hat Q^i_{\mathrm{\algo}}(s,a) \geq  \mathcal{T}_{\mathrm{In}} Q^i_{\mathrm{In}}(s,a) 
= Q^{i+1}_{\mathrm{In}}(s,a) 
\label{app_eq:over-e proof4}
\end{align}

Combining \cref{app_eq:over-e proof2,app_eq:over-e proof4}, for any $s,a \sim \Dcal$, we have
\begin{align}
\hat Q^{i+1}_{\mathrm{\algo}}(s,a) \geq Q^{i+1}_{\mathrm{In}}(s,a).
\end{align}

Hence, \cref{app_eq:over-e} still holds when $k=i+1$:
\begin{equation}
Q^{i+1}_{\mathrm{In}}(s,a) \leq \hat Q^{i+1}_{\mathrm{\algo}}(s,a) \leq Q^{i+1}_{\mathrm{In}}(s,a) + \frac{\lambda\epsilon_a K_Q \gamma}{1-\gamma}(1-\gamma^{i+1}), ~\forall s,a \sim \Dcal.
\end{equation}

Therefore, \cref{app_eq:over-e} holds for all $k \in \mathbb{Z}^+$, which concludes the proof.
\end{proof}

Since in-sample training eliminates extrapolation error completely~\cite{kostrikov2022offline,zhang2023insample}, $Q^k_{\mathrm{In}}$ can be considered a relatively accurate estimate. Therefore, \cref{app_prop:Limited over-estimation} indicates that \algo has limited over-estimation under the worst generalization case. Moreover, the bound gets tighter as $\epsilon_a$ gets smaller (more mild action generalization) and $\lambda$ gets smaller (more mild generalization propagation). This is consistent with our intuitions in \cref{sec:Doubly mild generalization}.

Finally, \cref{prop:Performance lower bound} in the main paper shows that even under worst-case generalization, \algo is guaranteed to output a safe policy with a performance lower bound.

We give a lemma before we start the proof of \cref{prop:Performance lower bound}, 
\begin{lemma}
\label{app_lem:TV}
Let $\pi_1$ and $\pi_2$ be two deterministic policies. Under \cref{app_ass:lP1}, the following inequality holds:
\begin{equation}
\mathrm{TV}\left(d^{\pi_1} || d^{\pi_2}\right) \leq C K_P \max_s \| \pi_1(s) -\pi_2(s) \|
\end{equation}
where $C$ is a positive constant and $d^\pi(s)$ is the state occupancy induced by $\pi$.
\begin{equation}
d^{\pi}(s)=(1-\gamma) \sum_{t=0}^{\infty} \gamma^{t}  \mathbb{E}_\pi \left[ \mathbb{I}\left[s_{t}=s\right]\right].
\end{equation}
\end{lemma}

\begin{proof}
Please refer to Lemma A.5 in \cite{ran2023policy} and Lemma 1 in \cite{xiong2022deterministic}.
\end{proof}

\begin{theorem}[Performance lower bound, \cref{prop:Performance lower bound}]
\label{app_prop:Performance lower bound}
Let $\hat\pi_{\mathrm{\algo}}$ be the learned policy of \algo by iterating $\hat{\mathcal{T}}_{\mathrm{\algo}}$, $\pi^*$ be the optimal policy, and $\epsilon_{\Dcal}$ be the inherent performance gap of the in-sample optimal policy $\epsilon_{\Dcal}:=J(\pi^*)-J(\pi^*_{\mathrm{In}})$. Under Assumptions \ref{app_ass:lQ1} and \ref{app_ass:lP1}, for sufficiently small $\epsilon_a$, we have
\begin{equation}
J(\hat{\pi}_{\mathrm{\algo}}) \geq J(\pi^*) - \frac{CK_P\Rmax}{1-\gamma}\epsilon_a - \epsilon_{\Dcal}.
\end{equation}
where $C$ is a positive constant.
\end{theorem}

\begin{proof}
Following previous works~\cite{kumar2019stabilizing,wu2022supported,kostrikov2022offline,mao2023supported}, we define the in-sample optimal policy as $\pi^*_{\mathrm{In}}$:
\begin{equation}
\pi^*_{\mathrm{In}}(s) = \argmax_{a\sim \hat\beta(\cdot|s)} Q^*_{\mathrm{In}}(s,a)
\end{equation}

We also use $\epsilon_{\Dcal}$ to denote the performance gap between the in-sample optimal policy and the globally optimal policy, which is fixed once the dataset is provided.
\begin{equation}
\epsilon_{\Dcal} =J(\pi^*)-J(\pi^*_{\mathrm{In}}).
\end{equation}

We use $\hat Q_{\mathrm{\algo}}$ to denote the learned Q function of \algo  with sufficient iteration steps $\hat Q_{\mathrm{\algo}}^k$, $k \rightarrow \infty$.
And $\hat{\pi}_{\mathrm{\algo}}$ is the output policy of $\hat Q_{\mathrm{\algo}}$:
\begin{equation}
\label{app_eq:perf lb proof6}
\hat{\pi}_{\mathrm{\algo}}(s) = \argmax_{a\sim \tilde\beta(\cdot|s)} \hat Q_{\mathrm{\algo}}(s,a)
\end{equation}

It holds that
\begin{align}
&|J(\pi^*) - J(\hat{\pi}_{\mathrm{\algo}})| \nonumber\\
=& |J(\pi^*) - J(\pi^*_{\mathrm{In}}) + J(\pi^*_{\mathrm{In}}) -J(\hat{\pi}_{\mathrm{\algo}})| \nonumber\\
\leq & |J(\pi^*) - J(\pi^*_{\mathrm{In}})| + |J(\pi^*_{\mathrm{In}}) -J(\hat{\pi}_{\mathrm{\algo}})| \nonumber\\
= & \epsilon_{\Dcal} + |J(\pi^*_{\mathrm{In}}) -J(\hat{\pi}_{\mathrm{\algo}})| \label{app_eq:perf lb proof}
\end{align}

In the following, we bound the term $|J(\pi^*_{\mathrm{In}}) -J(\hat{\pi}_{\mathrm{\algo}})|$.
\begin{align}
&|J(\pi^*_{\mathrm{In}}) -J(\hat{\pi}_{\mathrm{\algo}})| \nonumber\\
=&\left|\frac{1}{1-\gamma} \mathbb{E}_{s \sim d^{\hat{\pi}_{\mathrm{\algo}}}}[r(s)]-\frac{1}{1-\gamma} \mathbb{E}_{s \sim d^{\pi^*_{\mathrm{In}}}}[r(s)]\right| \nonumber\\
=&\frac{1}{1-\gamma}\left|\sum_{s}\left(d^{\hat{\pi}_{\mathrm{\algo}}}(s)-d^{\pi^*_{\mathrm{In}}}(s)\right) r(s)\right| \nonumber\\
\leq&\frac{1}{1-\gamma}\sum_{s}\left|\left(d^{\hat{\pi}_{\mathrm{\algo}}}(s)-d^{\pi^*_{\mathrm{In}}}(s)\right)\right| \left|r(s)\right| \nonumber\\
\leq& \frac{R_{\max }}{1-\gamma} \mathrm{TV}\left(d^{\hat{\pi}_{\mathrm{\algo}}}(s) || d^{\pi^*_{\mathrm{In}}}(s)\right) \nonumber\\
\leq& \frac{R_{\max }}{1-\gamma} C K_P \max_s \| \hat{\pi}_{\mathrm{\algo}}(s) -\pi^*_{\mathrm{In}}(s) \| \label{app_eq:perf lb proof1}
\end{align}
where the last inequality holds by \cref{app_lem:TV}.

According to \cref{app_prop:Limited over-estimation}, $\hat Q_{\mathrm{\algo}}$ satisfies the following inequality: 
\begin{equation}
\label{app_eq:perf lb proof2}
Q^*_{\mathrm{In}}(s,a) \leq \hat Q_{\mathrm{\algo}}(s,a) \leq Q^*_{\mathrm{In}}(s,a) + \frac{\lambda\epsilon_a K_Q \gamma}{1-\gamma},~\forall s,a \sim \Dcal.
\end{equation}

It means that for any $(s,a) \sim \Dcal$, with sufficiently small $\epsilon_a$, $\hat Q_{\mathrm{\algo}}(s,a)$ sufficiently approximates $Q^*_{\mathrm{In}}(s,a)$.
By \cref{app_def:mg}, $\tilde \beta$ is a mildly generalized policy. That is, for any $s \sim \Dcal$, $\tilde \beta$ satisfies
\begin{equation*}
\mathrm{supp}(\hat \beta(\cdot|s)) \subseteq \mathrm{supp}(\tilde \beta(\cdot|s)), ~~\text{and}~~ \max_{a_1 \sim \tilde \beta(\cdot|s)} \min_{a_2 \sim \hat \beta(\cdot|s)} \|a_1-a_2\| \leq \epsilon_a,
\end{equation*}

As $\hat{\pi}_{\mathrm{\algo}}(s) \in \tilde \beta(\cdot|s)$, it implies that we can find $a_{\mathrm{in}} \in \hat \beta(\cdot|s)$ (in dataset) such that $\|\hat{\pi}_{\mathrm{\algo}}(s)-a_{\mathrm{in}}\| \leq \epsilon_a$. 

Now suppose $a_{\mathrm{in}}$ is not the maximum point of $Q^*_{\mathrm{In}}(s, \cdot)$ at a certain $s$. We use $\pi^*_{\mathrm{In}}(s)$ to denote the maximum point of $Q^*_{\mathrm{In}}(s, \cdot)$. Let $\epsilon_{Q^*_{\mathrm{In}}}$ be the gap between $Q^*_{\mathrm{In}}(s, a_{\mathrm{in}})$ and $Q^*_{\mathrm{In}}(s, \pi^*_{\mathrm{In}}(s))$:
\begin{equation}
\label{app_eq:perf lb proof3}
\epsilon_{Q^*_{\mathrm{In}}}(s) := Q^*_{\mathrm{In}}(s, \pi^*_{\mathrm{In}}(s)) - Q^*_{\mathrm{In}}(s, a_{\mathrm{in}}) >0.
\end{equation}

By \cref{app_ass:lQ1} (Lipschitz $Q$), we have
\begin{equation}
\label{app_eq:perf lb proof4}
\hat Q_{\mathrm{\algo}}(s,\hat{\pi}_{\mathrm{\algo}}(s)) - \hat Q_{\mathrm{\algo}}(s, a_{\mathrm{in}}) \leq K_Q \|\hat{\pi}_{\mathrm{\algo}}(s)-a_{\mathrm{in}}\| \leq  K_Q\epsilon_a.
\end{equation}

Therefore,
\begin{align*}
& \hat Q_{\mathrm{\algo}}(s,\pi^*_{\mathrm{In}}(s))-\hat Q_{\mathrm{\algo}}(s,\hat{\pi}_{\mathrm{\algo}}(s)) \\
\geq & \hat Q_{\mathrm{\algo}}(s,\pi^*_{\mathrm{In}}(s))- \hat Q_{\mathrm{\algo}}(s, a_{\mathrm{in}}) - K_Q\epsilon_a \\
\geq & Q^*_{\mathrm{In}}(s,\pi^*_{\mathrm{In}}(s)) - Q^*_{\mathrm{In}}(s, a_{\mathrm{in}}) - \frac{\lambda\epsilon_a K_Q \gamma}{1-\gamma} - K_Q\epsilon_a \\
= & \epsilon_{Q^*_{\mathrm{In}}}(s) - \frac{\lambda\epsilon_a K_Q \gamma}{1-\gamma} - K_Q\epsilon_a
\end{align*}
where the first inequality holds by \cref{app_eq:perf lb proof4}, the second inequality holds by \cref{app_eq:perf lb proof2}, and the last equality holds by \cref{app_eq:perf lb proof3}.

Hence, for sufficiently small $\epsilon_a$ such that $\epsilon_{Q^*_{\mathrm{In}}}(s) - \frac{\lambda\epsilon_a K_Q \gamma}{1-\gamma} - K_Q\epsilon_a > 0$, i.e.,
\begin{equation}
\epsilon_a < \frac{(1-\gamma)\epsilon_{Q^*_{\mathrm{In}}}(s)}{K_Q(1-\gamma+\lambda\gamma)},
\end{equation}
it holds that $\hat Q_{\mathrm{\algo}}(s,\pi^*_{\mathrm{In}}(s))-\hat Q_{\mathrm{\algo}}(s,\hat{\pi}_{\mathrm{\algo}}(s)) > 0$.
As $\pi^*_{\mathrm{In}}(s) \in \hat \beta(\cdot|s)$, it also satisfies $\pi^*_{\mathrm{In}}(s) \in \tilde \beta(\cdot|s)$.
This contradicts the definition of  $\hat{\pi}_{\mathrm{\algo}}(s)$ in \cref{app_eq:perf lb proof6}:
\begin{equation*}
\hat{\pi}_{\mathrm{\algo}}(s) = \argmax_{a\sim \tilde\beta(\cdot|s)} \hat Q_{\mathrm{\algo}}(s,a)
\end{equation*}

Therefore, $a_{\mathrm{in}}$ is the maximum point of $Q^*_{\mathrm{In}}(s, \cdot)$. In other words, the maximum point of $Q^*_{\mathrm{In}}(s, \cdot)$ (denoted by $\pi^*_{\mathrm{In}}(s)$) is the closest neighbor of $\hat{\pi}_{\mathrm{\algo}}(s)$ in the dataset ($\hat\beta(\cdot|s)>0$):
\begin{equation*}
\pi^*_{\mathrm{In}}(s) = \argmin_{a \sim \hat\beta{(\cdot|s)}} \|a-\hat{\pi}_{\mathrm{\algo}}(s)\|
\end{equation*}

As $\hat{\pi}_{\mathrm{\algo}}(s) \in \tilde \beta(\cdot|s)$, the following inequality holds by \cref{app_def:mg}:
\begin{equation*}
\|\hat{\pi}_{\mathrm{\algo}}(s)-\pi^*_{\mathrm{In}}(s)\| \leq \epsilon_a.
\end{equation*}

Therefore, we have
\begin{equation}
\label{app_eq:perf lb proof5}
|J(\pi^*_{\mathrm{In}}) -J(\hat{\pi}_{\mathrm{\algo}})| \leq \frac{R_{\max }}{1-\gamma} C K_P \epsilon_a.
\end{equation}

By combining \cref{app_eq:perf lb proof,app_eq:perf lb proof5}, we have
\begin{equation}
J(\hat{\pi}_{\mathrm{\algo}}) \geq J(\pi^*) - \frac{CK_P\Rmax}{1-\gamma}\epsilon_a - \epsilon_{\Dcal}.
\end{equation}

This concludes the proof.
\end{proof}

\section{Experimental Details}
\label{app_sec:experimental_details}

\subsection{Experimental Details in Offline Experiments}
\label{app_sec:experimental_details offline}

\begin{table}[htbp]
\caption{Hyperparameters of \algo.}

\label{app_tab:hyper}
\begin{center}
\begin{tabular}{cll}
\toprule
                              & Hyperparameter          & \multicolumn{1}{l}{Value}           \\ \midrule
\multirow{11}{*}{\algo}         & Optimizer               & \multicolumn{1}{l}{Adam~\citep{kingma2014adam}}            \\
                              & Critic learning rate    & \multicolumn{1}{l}{$3\times 10^{-4}$}            \\
                              & Actor learning rate     & \multicolumn{1}{l}{$3\times 10^{-4}$ with cosine schedule}  \\
                              & Batch size              & 256                                 \\
                              & Discount factor                & 0.99                                \\
                              & Number of iterations    & $10^6$                             \\
                              & Target update rate      & 0.005                               \\
                              & Number of Critics & 2                                   \\
                              & Penalty coefficient $\nu$  & \multicolumn{1}{l}{ \{0.1,10\} for Gym-MuJoCo} \\
                              &                         & \multicolumn{1}{l}{\{0.5\} for Antmaze}    \\
                              & Mixture coefficient $\lambda$  & 0.25  \\
                              \midrule
\multirow{4}{*}{IQL Specific} & Expectile $\tau$   & \multicolumn{1}{l}{0.7 for Gym-MuJoCo} \\
                              &                         & \multicolumn{1}{l}{0.9 for Antmaze}    \\
                              & Inverse temperature $\alpha$ & \multicolumn{1}{l}{3.0 for Gym-MuJoCo} \\
                              &                         & \multicolumn{1}{l}{10.0 for Antmaze}    \\
                              \midrule
\multirow{2}{*}{Architecture} & Actor    & input-256-256-output                                 \\
                              & Critic & input-256-256-1                                      \\ \bottomrule
\end{tabular}
\end{center}
\end{table}

Our evaluation criteria follow those used in most previous works. For the Gym locomotion tasks, we average returns over 10 evaluation trajectories and 5 random seeds, while for the AntMaze tasks, we average over 100 evaluation trajectories and 5 random seeds. Following the suggestions in the benchmark~\cite{fu2020d4rl}, we subtract 1 from the rewards for the AntMaze datasets. And following previous works~\cite{fujimoto2021minimalist,kostrikov2022offline,wu2022supported,xu2023offline}, we normalize the states in Gym locomotion datasets. We choose TD3~\citep{fujimoto2018addressing} as our base algorithm and optimize a deterministic policy. Thus we replace the log likelihood in Eq.~\eqref{eq:pi} with mean squared error in practice, which is equivalent to optimizing a Gaussian policy with fixed variance~\citep{fujimoto2021minimalist}.
The reported results are the normalized scores, which are offered by the D4RL benchmark~\cite{fu2020d4rl} to measure how the learned policy compared with random and expert policy:
\begin{equation*}
\text{D4RL score} = 100 \times \frac{\text{learned policy return}-\text{random policy return}}{\text{expert policy return}-\text{random policy return}}
\end{equation*}

As we implement our main algorithm based on IQL~\cite{kostrikov2022offline}, we use the hyperparameters suggested in their paper for fair comparisons, i.e., $\tau=0.7$ and $\alpha=3$ for Gym locomotion tasks and $\tau=0.9$ and $\alpha=10$ for AntMaze tasks. For the results of $\mathcal X$QL+\algo and SQL+\algo, we also adopt the suggested hyperparameters in their papers~\cite{garg2023extreme,xu2023offline} for fair comparisons. In detail, we choose $\beta$ in $\mathcal X$QL~\cite{garg2023extreme} as $5.0$ in medium, medium-replay, and medium-expert datasets, and $\alpha$ in SQL~\cite{xu2023offline} as $2.0$ for medium, medium-replay datasets, and $5.0$ for medium-expert datasets.

\algo has two main hyperparameters: mixture coefficient $\lambda$ and penalty coefficient $\nu$. We use $\lambda=0.25$ for all tasks. We use $\nu=0.5$ for Antmaze tasks and $\nu \in \{0.1, 10\}$ for Gym locomotion tasks ($0.1$ for medium, medium-replay, random datasets; $10$ for expert and medium-expert datasets). All hyperparameters of \algo are included in \cref{app_tab:hyper}.

\subsection{Experimental Details in Offline-to-online Experiments}
\label{app_sec:experimental_details online}
For online fine-tuning experiments, we first run offline RL for $1\times10^6$ gradient steps. Then we continue training while collecting data actively in the environment and adding the data to the replay buffer. We perform online fine-tuning for $1\times10^6$ steps with $1$ update-to-data (UTD) ratio, and collect data with exploration noise $0.1$ as suggested by TD3~\cite{fujimoto2018addressing}. During offline pre-training, we fix the mixture coefficient $\lambda=0.25$ and the penalty coefficient $\nu=0.5$, while in the online phase, we exponentially adjust $\lambda$ and $\nu$, as \algo with $\lambda=1$ and $\nu=0$ corresponds to standard online RL. In the challenging AntMaze domains characterized by high-dimensional state and action spaces, as well as sparse rewards, the extrapolation error remains significant even during the online phase~\cite{wu2022supported}. Therefore, we decay $\lambda$ from $0.25$ to $0.5$ and $\nu$ from $0.5$ to $0.005$ ($1\%$ of its initial value), employing a decay rate of $0.99$ every $1000$ gradient steps. Additionally, following previous works~\cite{wu2022supported,tarasov2024revisiting}, we set $\gamma = 0.995$ when fine-tuning on antmaze-large datasets, for both \algo and IQL to ensure a fair comparison. All other training details remain consistent between the offline RL phase and the online fine-tuning phase.

\section{Additional Experimental Results}
\label{app_sec:experimental_results}

\subsection{Computational Cost}
\label{app_sec:computational cost}

We test the runtime of offline RL algorithms on halfcheetah-medium-replay-v2 on a GeForce RTX 3090. The results of \algo and other baselines are shown in \cref{app_fig:runtime}. It takes 1.7h for \algo to finish the task, which is comparable to the fastest offline RL algorithm TD3BC~\citep{fujimoto2021minimalist}. 

\begin{figure}[h]
    \centering
    \includegraphics[scale=0.6]{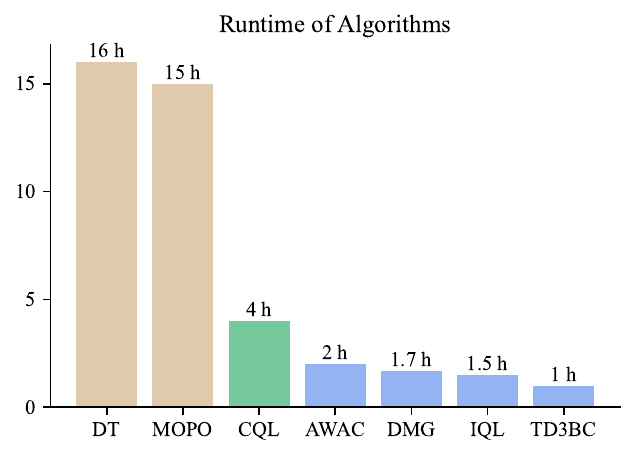}
    \caption{Runtime of algorithms on halfcheetah-medium-replay-v2 on a GeForce RTX 3090.}
    \label{app_fig:runtime}
\end{figure}

\subsection{Offline Training Results of DMG on More Random Seeds}
\label{app_sec:10 seeds}

The experimental results in the main paper show the mean and standard deviation (SD) over five random seeds. According to \citep{patterson2023empirical}, we conduct experiments to test \algo on additional random seeds, reporting 95\% confidence interval (CI) over 10 random seeds. \cref{app_tab:10 seeds} shows the comparison between the new results (10seeds/95\%CI) and the previously reported results (5seeds/SD in \cref{tab:d4rl}) on the D4RL offline training tasks. The results show that our method achieves about the same performance as under the previous evaluation criterion.

\begin{table}[h]
\caption{Comparison of DMG under different evaluation criteria on D4RL offline training tasks.}
\vspace{-1mm}
\label{app_tab:10 seeds}
\begin{center}
\begin{small}
\begin{tabular}{lcc}
\toprule
Dataset-v2 & DMG (5seeds/SD) & DMG (10seeds/95\%CI) \\ \midrule
halfcheetah-m & \textbf{54.9$\pm$0.2} & \textbf{54.9$\pm$0.3} \\
hopper-m & \textbf{100.6$\pm$1.9} & 100.5$\pm$1.0 \\
walker2d-m & \textbf{92.4$\pm$2.7} & 92.0$\pm$1.2 \\
halfcheetah-m-r & \textbf{51.4$\pm$0.3} & \textbf{51.4$\pm$0.4} \\
hopper-m-r & 101.9$\pm$1.4 & \textbf{102.1$\pm$0.6} \\
walker2d-m-r & 89.7$\pm$5.0 & \textbf{90.3$\pm$2.8} \\
halfcheetah-m-e & 91.1$\pm$4.2 & \textbf{92.9$\pm$2.1} \\
hopper-m-e & \textbf{110.4$\pm$3.4} & 109.0$\pm$2.6 \\
walker2d-m-e & \textbf{114.4$\pm$0.7} & 113.9$\pm$1.2 \\
halfcheetah-e & \textbf{95.9$\pm$0.3} & \textbf{95.9$\pm$0.2} \\
hopper-e & 111.5$\pm$2.2 & \textbf{111.8$\pm$1.3} \\
walker2d-e & \textbf{114.7$\pm$0.4} & 114.5$\pm$0.3 \\
halfcheetah-r & \textbf{28.8$\pm$1.3} & 28.7$\pm$1.2 \\
hopper-r & 20.4$\pm$10.4 & \textbf{21.6$\pm$6.6} \\
walker2d-r & 4.8$\pm$2.2 & \textbf{7.7$\pm$3.0} \\
\midrule
locomotion total & 1182.8 & \textbf{1187.2} \\
\midrule
antmaze-u & \textbf{92.4$\pm$1.8} & 91.8$\pm$1.6 \\
antmaze-u-d & \textbf{75.4$\pm$8.1} & 73.0$\pm$5.0 \\
antmaze-m-p & 80.2$\pm$5.1 & \textbf{80.5$\pm$2.1} \\
antmaze-m-d & \textbf{77.2$\pm$6.1} & 76.7$\pm$3.6 \\
antmaze-l-p & 55.4$\pm$6.2 & \textbf{56.7$\pm$3.6} \\
antmaze-l-d & \textbf{58.8$\pm$4.5} & 57.2$\pm$2.7 \\
\midrule
antmaze total & \textbf{439.4} & 435.9 \\
\bottomrule
\end{tabular}
\end{small}
\end{center}
\end{table}

\subsection{Learning Curves of \algo during Offline Training}
\label{app_sec:offline curves}
Learning curves during offline training on Gym-MuJoCo locomotion tasks and Antmaze tasks are presented in \cref{app_fig:mujoco_appendix} and \cref{app_fig:antmaze_appendix}, respectively. The curves are averaged over 5 random seeds, with the shaded area representing the standard deviation across seeds.

\subsection{Learning Curves of \algo during Online Fine-tuning}
\label{app_sec:online curves}
Learning curves during online fine-tuning on Antmaze tasks are presented in \cref{app_fig:finetune_appendix}. The curves are averaged over 5 random seeds, with the shaded area representing the standard deviation across seeds.

\section{Broader Impact}
\label{app_sec:Broader Impact}

Offline reinforcement learning (RL) presents a promising avenue for enhancing and broadening the practical applicability of RL across various domains including robotics, recommendation systems, healthcare, and education, characterized by costly or hazardous data collection processes. However, it is imperative to recognize the potential adverse societal ramifications associated with any offline RL algorithm. One such concern pertains to the possibility that the offline data utilized for training may harbor inherent biases, which could subsequently permeate into the acquired policy. Furthermore, it is essential to contemplate the potential implications of offline RL on employment, given its contribution to automating tasks conventionally executed by human experts, such as factory automation or autonomous driving. Addressing these challenges is essential for fostering the responsible development and deployment of offline RL algorithms, with the aim of maximizing their positive impact while mitigating negative societal consequences.

From an academic perspective, this research scrutinizes offline RL through the lens of generalization, balancing the need for generalization with the risk of over-generalization. The proposed approach \algo potentially offers researchers a new perspective on appropriately exploiting generalization in offline RL. Besides, \algo also holds the promise to be extended to safe RL~\citep{achiam2017constrained,gu2022review,garcia2015comprehensive}, multi-agent RL~\citep{lowe2017multi,rashid2020monotonic,shao2023complementary,qu2024choices,gronauer2022multi}, and meta RL~\citep{finn2017model,wang2022model,wang2024simple,wang2024robust,beck2023survey}.

\begin{figure}[h]
	\centering
        \vspace{1cm}
	\includegraphics[width=\linewidth]{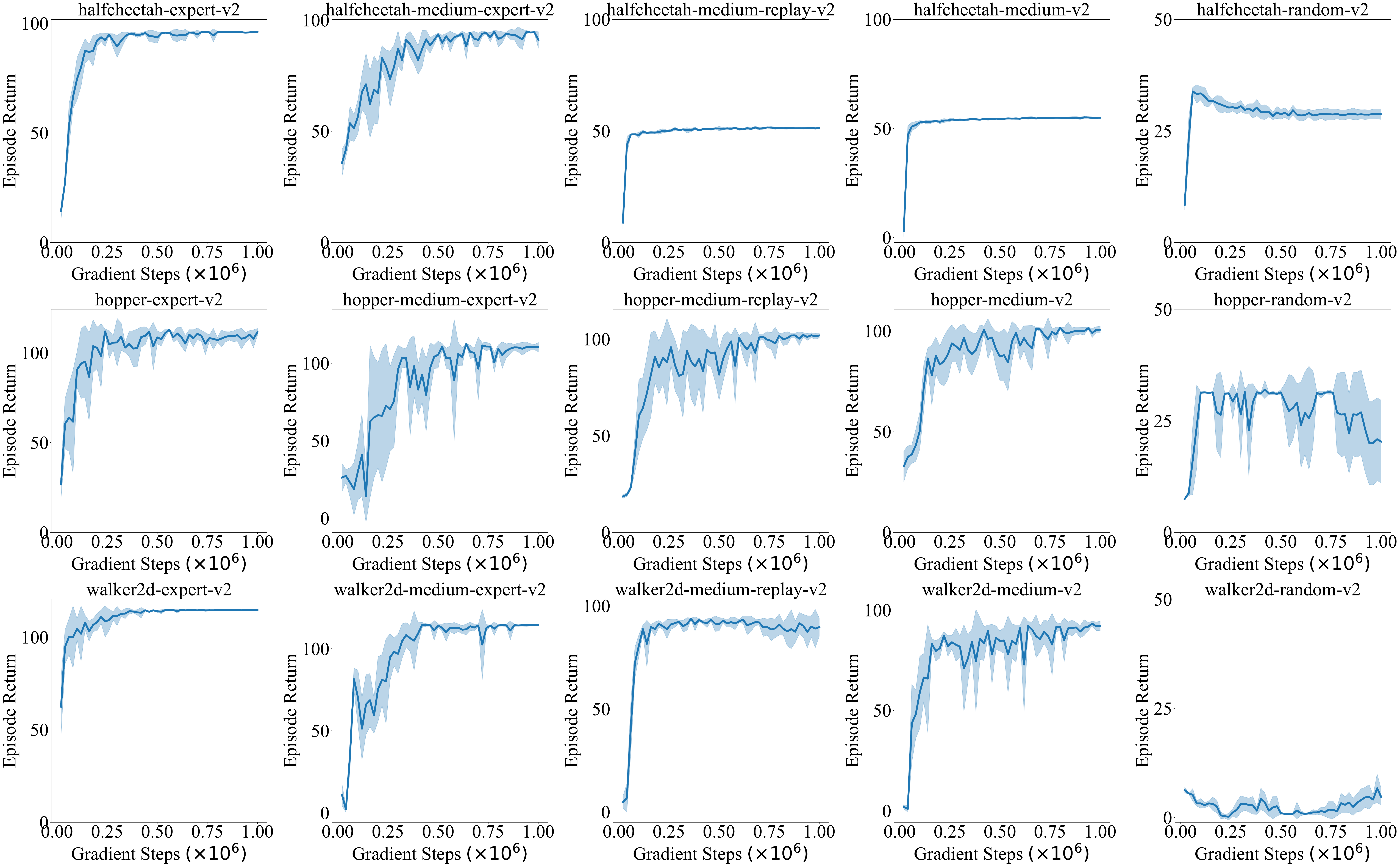}
	\vspace{-0.0cm}
	\caption{Learning curves of \algo on Gym locomotion tasks during offline training.
        The curves are averaged over 5 random seeds, with the shaded area representing the standard deviation across seeds.
	}
	\label{app_fig:mujoco_appendix}
\end{figure}
\begin{figure}
	\centering
	\includegraphics[width=0.92\linewidth]{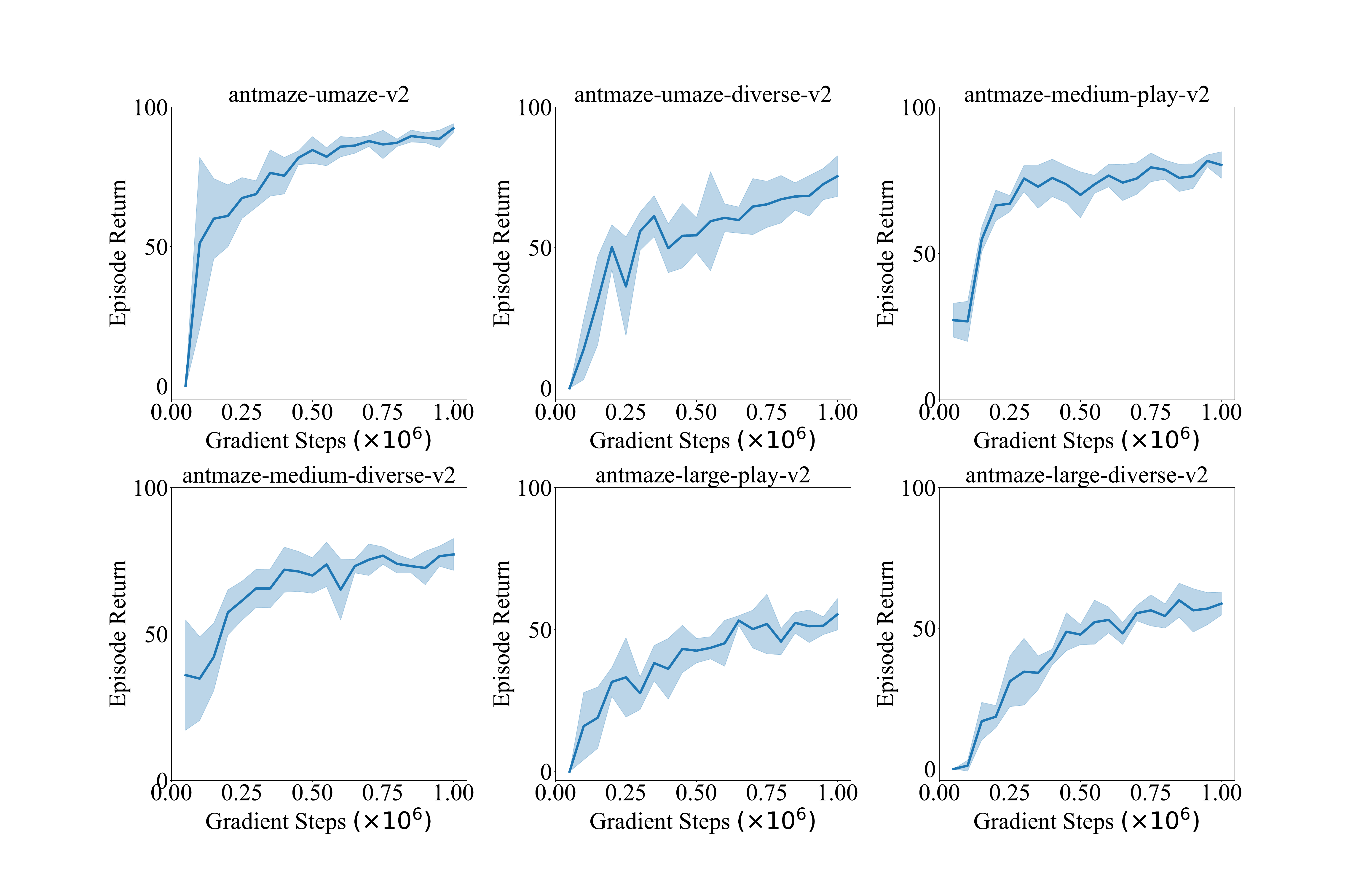}
	\vspace{-0.0cm}
	\caption{Learning curves of \algo on Antmaze tasks during offline training.
        The curves are averaged over 5 random seeds, with the shaded area representing the standard deviation across seeds.
	}
	\label{app_fig:antmaze_appendix}
\end{figure}
\begin{figure}
	\centering
	\includegraphics[width=0.92\linewidth]{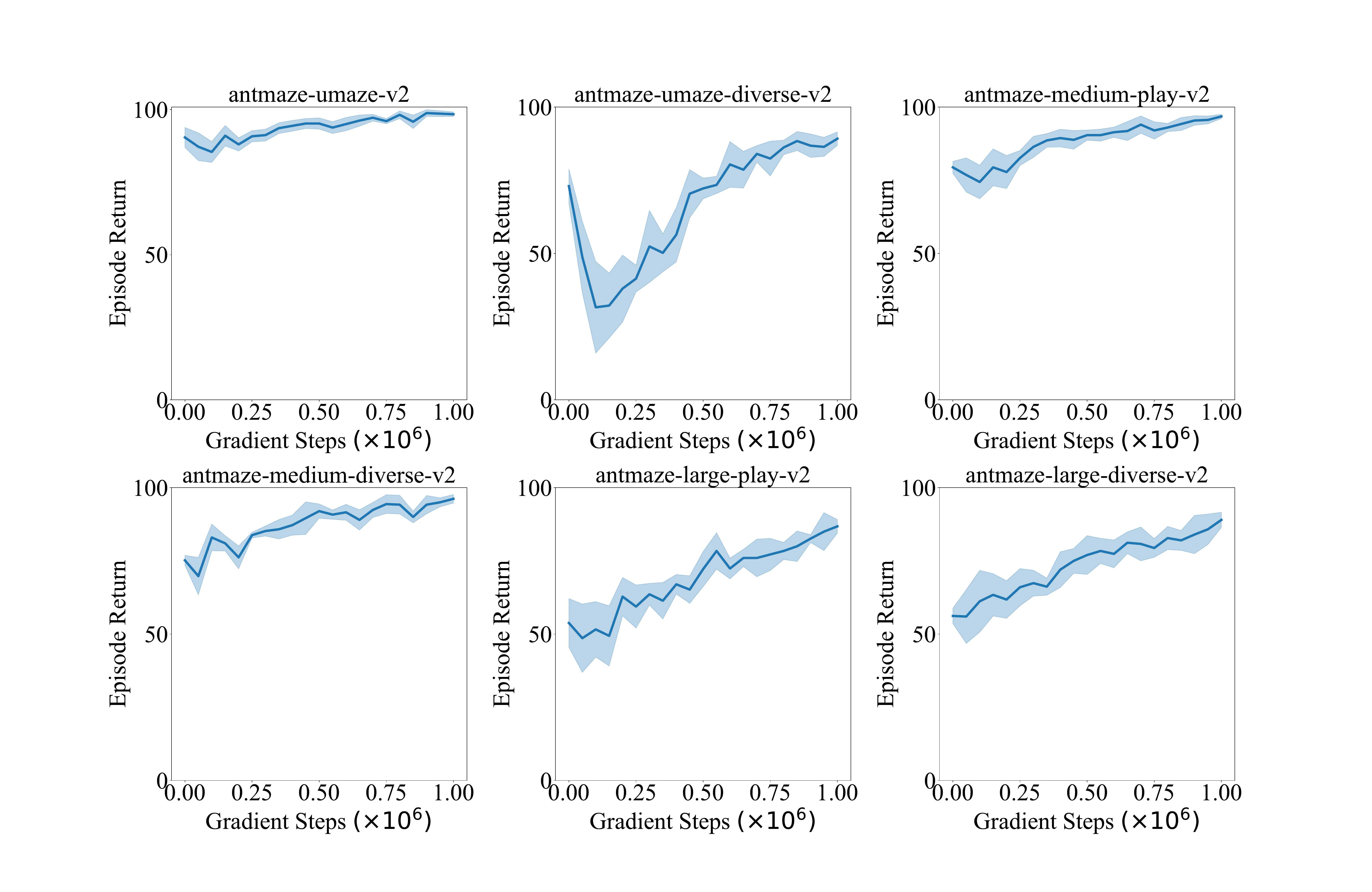}
	\vspace{-0.0cm}
	\caption{Learning curves of \algo on Antmaze tasks during online fine-tuning.
        The curves are averaged over 5 random seeds, with the shaded area representing the standard deviation across seeds.
	}
	\label{app_fig:finetune_appendix}
\end{figure}

\end{document}